\documentclass{article}

     \PassOptionsToPackage{round, sort&compress}{natbib}

     \usepackage[final]{neurips_2023}

\usepackage[utf8]{inputenc} 
\usepackage[T1]{fontenc}    
\usepackage{hyperref}       
\usepackage{url}            
\usepackage{booktabs}       
\usepackage{amsfonts}       
\usepackage{nicefrac}       
\usepackage{microtype}      
\usepackage{xcolor}         
\usepackage{caption}
\usepackage{subcaption}
\usepackage{tikz}
\usepackage{amsmath}
\usepackage[algo2e]{algorithm2e}
\usepackage{algorithm,algpseudocode}
\usepackage{amsthm}
\usepackage{comment}
\usepackage{dsfont}
\usepackage{wrapfig}
\tikzstyle{state}=[
    circle,
    minimum size = 10mm,
    draw=black,
    very thick,
]
\tikzstyle{midstate}=[
    circle,
    minimum size = 1mm,
    draw=black,
    thick,
]
\tikzstyle{smallstate}=[
minimum size = 0mm,
]
\tikzset{
    point/.style = {circle, draw, inner sep=0.05cm,fill,node contents={}},
}
\usetikzlibrary{decorations.text, positioning, fit, bayesnet}
\usepackage{amsthm}
\usepackage{comment}
\newcommand{\ind}{\perp\!\!\!\!\perp} 
\newtheorem{theorem}{Theorem}
\newtheorem{corollary}{Corollary}

\newtheorem{lemma}{Lemma}
\newtheorem{definition}{Definition}

\newcommand{\Xink}{X_{i;n}^{k}}
\newcommand{\UPink}{UP_{i;n}^{k}}
\newcommand{\Xinkstar}{X_{i;n}^{k, *}}
\newcommand{\vecXin}{$\{(X_{i;n}, X_{i+1;n}, \ldots, X_{d;n})\}_{n \in \mathbb{N}}$}

\SetKwInput{KwInput}{Input}                
\SetKwInput{KwOutput}{Output}              
\SetKwInput{KwStepone}{Step 1}
\SetKwInput{KwSteptwo}{Step 2}
\SetKwInput{KwStepthree}{Step 3}
\usepackage{enumitem}
\usepackage{cleveref}
\usepackage{thmtools}
\usepackage{thm-restate}
\hypersetup{
    colorlinks,
    linkcolor={red},
    citecolor={black},
    urlcolor={gray}
}
\usepackage{xr}
\setlist[description]{leftmargin=1em, itemindent=-1em}

\title{Causal de Finetti: On the Identification of Invariant Causal Structure in Exchangeable Data}

\author{%
Siyuan Guo$^{12*}$ \quad Viktor Tóth$^{1*}$ \quad Bernhard Schölkopf$^2$ \quad Ferenc Huszár$^1$ \\
$^1$University of Cambridge \quad $^2$Max Planck Institute for Intelligent Systems\\
\texttt{\{syg26,fh277\}@cam.ac.uk} \quad \texttt{toth.viktor7400@gmail.com} \\
\texttt{bs@tuebingen.mpg.de} \\
}

\begin{document}

\crefname{section}{§}{§§}

\maketitle
\def\thefootnote{*}\footnotetext{Equal contribution}

\begin{abstract}
Constraint-based causal discovery methods leverage conditional independence tests to infer causal relationships in a wide variety of applications. Just as the majority of machine learning methods, existing work focuses on studying \textit{independent and identically distributed} data. However, it is known that even with infinite i.i.d.\ data, constraint-based methods can only identify causal structures up to broad Markov equivalence classes, posing a fundamental limitation for causal discovery. In this work, we observe that exchangeable data contains richer conditional independence structure than i.i.d.\ data, and show how the richer structure can be leveraged for causal discovery. We first present causal de Finetti theorems, which state that exchangeable distributions with certain non-trivial conditional independences can always be represented as \textit{independent causal mechanism (ICM)} generative processes. We then present our main identifiability theorem, which shows that given data from an ICM generative process, its unique causal structure can be identified through performing conditional independence tests. We finally develop a causal discovery algorithm and demonstrate its applicability to inferring causal relationships from multi-environment data. Our code and models are publicly available at: \url{https://github.com/syguo96/Causal-de-Finetti}

\end{abstract}

\section{Introduction}

Learning causal structure from observational data is a key step towards causally robust predictions in machine learning. Most existing causal discovery theory \citep{Pearl2009Causality:Inference.} focuses on the study of \textit{independent and identically distributed (i.\,i.\,d.)} data. Indeed, a majority of practical methods \citep{spirtes2000causation, spirtes2000constructing, chickering2002optimal} based on i.i.d.\ data only allows us to determine causal structure up to broad equivalence classes, and going beyond that is known to be impossible without further constraints \citep{Pearl2009Causality:Inference.}. For example, the basic task of identifying a bivariate cause-effect relationship (i.e. $X$ causes $Y$ or $Y$ causes $X$) on i.i.d.\ data is known to be impossible. Current methods impose additional restrictions, e.g., linearity assumptions \citep{Shimizu2006ADiscovery, shimizu2011directlingam, hoyer2012causal} or assumptions about the additive nature of noise \citep{Hoyer2008NonlinearModels, Zhang2009OnModel, peters2014causal} to ensure identification. 

A more recent line of work relaxes the i.i.d.\ assumption and considers inferring causal structure from grouped or multi-environment data \citep{Peters2016CausalIntervals, Tian2001CausalChanges, Heinze-Deml2018InvariantModels, Rojas-Carulla2018InvariantModels,huang2020causal,Arjovsky2019InvariantMinimization}. Our key observation is that studying grouped data is akin to relaxing the assumption on the data generating process from i.i.d.\ to exchangeable. In this work, we study causal structure learning in exchangeable data and show unique causal structure identification is enabled by the richer conditional independence structure of exchangeable processes. 

Several works in multi-environment data implicitly leverage the \textit{independent causal mechanism (ICM)} assumption \citep{Janzing2010CausalCondition, Aldrich1989, Pearl2009Causality:Inference.}, which postulates that causal mechanisms of the true underlying generating process do not inform or influence one another. Despite having been widely applied \citep{Parascandolo2018, goyal2019recurrent, brehmer2022weakly, madan2021fast}, the ICM assumption is rarely stated any more formally than above, and thus lacks a statistical formalization. It is also unknown what the fundamental limitations of inferring causal structure under the ICM assumptions are. Our work makes three contributions to clarify these questions: \looseness=-1

\begin{itemize}[leftmargin=*]
\item{\textbf{Causal de Finetti theorems}} (\cref{sec:causal_de_finetti_bivariate}) provide a mathematical justification for the \textit{independent causal mechanism (ICM)} assumption in data generating processes. It states that any exchangeable process satisfying a certain set of conditional independence properties can be represented as a generative process in which factors in a Markov factorization are statistically independent. We show how causal de Finetti substantiates the ICM principle and call such models ICM-generative processes.
\item{\textbf{Our main identifiability theorem}} (\cref{sec:identifiability}), informally stated, shows that if data is sampled from an ICM generative process, the causal graph is uniquely identifiable by testing conditional independence relationships. \looseness=-1
\item{\textbf{Causal discovery in multi-environment data:}} Section \ref{sec:multi-env-data} connects the identifiability theorem for ICM generative models to the analysis of multi-environment data. This section establishes that multi-environment data can be viewed as observing finite marginals of i.i.d.\ copies of an exchangeable process. This then allows us to develop an algorithm for recovering causal structure from data coming from a sufficient number of environments.
\end{itemize}

Our work thus provides strong probabilistic justification for approaches based on the independent causal mechanisms assumption and algorithms that require non-i.i.d.\ grouped data. We review the use of exchangeability in causality and approaches for causal structure learning in grouped data in Section \ref{sec:related-work}. In Section \ref{sec:experiments}, we present experiments that validate our approach to inferring causal structure from multi-environment data using conditional independence testing. Fig. \ref{fig:iid_exchangeable_illustration} summarizes the main contributions of the paper and Fig. \ref{fig:illus_e_equiv} illustrates differences between data generated by causal graph under i.i.d.\ process and ICM-generative process. 

\begin{figure}
     \centering
     \begin{subfigure}[b]{0.3\textwidth}
        \centering
        \begin{tikzpicture}[line cap=round,line join=round,scale=.6,x=1cm,y=1cm]
            \fill[fill = {rgb,255:red,153; green,221; blue,255}, line width = 1, draw = black] (-30:1.5) -- (-30:3) arc (-30:90:3) -- (90:1.5) -- (90:1.5) arc (90:-30:1.5) node[midway, sloped, above = 0.1cm] {$\mathbf{X \rightarrow Y}$};
            \fill[fill = {rgb,255:red,238; green,136; blue,102}, line width = 1, draw = black] (90:1.5) -- (90:3) arc (90:210:3) -- (210:1.5) -- (210:1.5) arc (210:90:1.5) node[midway, sloped, above = 0.1cm] {$\mathbf{Y \rightarrow X}$};
            \fill[fill = {rgb,255:red,68; green,187; blue,153}, line width = 1, draw = black] (210:1.5) -- (210:3) arc (210:330:3) -- (330:1.5) arc (330:210:1.5) node[midway, sloped, below = 0.2cm] {$\mathbf{X \ind Y}$};
            \draw [line width = 1] (0,0) circle (1.5cm);
            \draw[postaction={decorate,decoration={raise= 2ex,text along path, reverse path, text align=center,text={|\bf\sffamily|ICM-generative process}}}] (3,0) arc(0:120:3);
            \draw[postaction={decorate,decoration={raise= -3ex,text along path, reverse path, text align=center,text={|\bf\sffamily|i.i.d}}}] (1.5,0) arc(0:90:1.5);
            \draw[fill = {rgb,255:red,238; green,221; blue,136}, line width = 1] (0,0) -- (-0.75*1.717, -0.75) arc (210:330:1.5cm) -- (0,0);
            \draw [fill=black] (0,0) circle (2.5pt);
        \end{tikzpicture}
        \caption{}
         \label{fig:iid_exchangeable_illustration}
     \end{subfigure}
     \hfill
     \begin{subfigure}[b]{0.65\textwidth}
         \centering
         \begin{tikzpicture}
            \node[midstate] (x1) at (4, 1.5) {$X_1$} ;
            \node[midstate] (y1) at (6, 1.5) {$Y_1$} ;
            \node[midstate] (x2) at (4, -1.1) {$X_2$} ;
            \node[midstate] (y2) at (6, -1.1) {$Y_2$} ;
            \node[rectangle, fill = gray!20, thick, draw = black] (latent) at (5, 0.2) {$\theta, \boldsymbol{\psi}$};
            \path[->] (x1) edge[line width = 1pt] (y1);
            \path[->] (x2) edge[line width = 1pt] (y2);
            \path[->] (latent) edge[line width = 1pt] (x1);
            \path[->] (latent) edge[line width = 1pt] (x2);
            \path[->] (latent) edge[line width = 1pt] (y1);
            \path[->] (latent) edge[line width = 1pt] (y2);
            \node at (5, 2.3) {\textbf{de Finetti}};
                
            \node[midstate] (x1) at (8, 1.5) {$X_1$} ;
            \node[midstate] (y1) at (10, 1.5) {$Y_1$} ;
            \node[midstate] (x2) at (8, -1.1) {$X_2$} ;
            \node[midstate] (y2) at (10, -1.1) {$Y_2$} ;
            \node[rectangle, fill = gray!20, thick, draw = black] (theta) at (8, 0.2) {$\theta$};
            \node[rectangle, fill = gray!20, thick, draw = black] (psi) at (10, 0.2) {$\boldsymbol{\psi}$};
            \path[->] (x1) edge[line width = 1pt] (y1);
            \path[->] (x2) edge[line width = 1pt] (y2);
            \path[->] (theta) edge[line width = 1pt] (x1);
            \path[->] (theta) edge[line width = 1pt] (x2);
            \path[->] (psi) edge[line width = 1pt] (y1);
            \path[->] (psi) edge[line width = 1pt] (y2);
            \node at (9, 2.3) {\textbf{Causal de Finetti}};
         \end{tikzpicture}
         \caption{}
         \label{fig:de-finetti-representation}
        \end{subfigure}
    \caption{(a) is an illustration showing how i.i.d.\ data and certain exchangeable data differ in identifying the correct causal structure for a bivariate model. Each quadrant represents a causal structure, ${X \ind Y}$, ${X \rightarrow Y}$, ${X \leftarrow Y}$. The inner circle represents i.i.d.\ regime and the outer circle represents certain exchangeable regime. Under i.i.d.\ data, one can only identify ${X \ind Y}$, whereas certain exchangeable data (i.e., ICM-generative processes) enables one to identify unique causal structures. (b) illustrates a differentiating factor between de Finetti and causal de Finetti's representation on exchangeable data. Causal de Finetti disentangles the latents and substantiates causal mechanisms are independent in the sense latent parameters governing each mechanisms are statistically independent.}
    \label{fig:iid_exchangeable_illustration}
\end{figure}

\begin{figure}
\centering
   \resizebox{.8\columnwidth}{!}{
\begin{tikzpicture}
    
    \node[midstate] (xi) at (-1, 1.5) {$X_i$};
    \node[midstate] (yi) at (1, 1.5) {$Y_i$};
    \path[->] (xi) edge[line width = 1pt] (yi);
    \plate [inner sep=.1cm,xshift=.02cm,yshift=.01cm] {plate0} {(xi)(yi)} {};
    \node at (0, 2.4) {\textbf{(a): i.i.d.\ process}};
    
    \node[midstate] (xe) at (-1, -1.2) {$X_i$};
    \node[midstate] (ye) at (1, -1.2) {$Y_i$};
    \node[rectangle, fill = gray!20, thick, draw = black] (theta) at (-1, 0) {$\theta$};
    \node[rectangle, fill = gray!20, thick, draw = black] (psi) at (1, 0) {$\boldsymbol{\psi}$};
    \path[->] (xe) edge[line width = 1pt] (ye);
    \path[->] (theta) edge[line width = 1pt] (xe);
    \path[->] (psi) edge[line width = 1pt] (ye);
    \plate [inner sep=.1cm,xshift=.02cm,yshift=.02cm] {plateE1} {(xe)(ye)} {};
    \plate [inner sep=.1cm,xshift=0.02cm,yshift=.02cm] {plateE2} {(plateE1)(theta)(psi)} {}; 
    \node at (0, -2.8) {\textbf{(b): exchangeable process}};
    
    \node[midstate] (x1) at (3, 1.5) {$X_1$} ;
    \node[midstate] (y1) at (5, 1.5) {$Y_1$} ;
    \node[midstate] (x2) at (3, -1.1) {$X_2$} ;
    \node[midstate] (y2) at (5, -1.1) {$Y_2$} ;
    \node[rectangle, fill = gray!20, thick, draw = black] (theta) at (3, 0.2) {$\theta$};
    \node[rectangle, fill = gray!20, thick, draw = black] (psi) at (5, 0.2) {$\boldsymbol{\psi}$};
    \path[->] (x1) edge[line width = 1pt] (y1);
    \path[->] (x2) edge[line width = 1pt] (y2);
    \path[->] (theta) edge[line width = 1pt] (x1);
    \path[->] (theta) edge[line width = 1pt] (x2);
    \path[->] (psi) edge[line width = 1pt] (y1);
    \path[->] (psi) edge[line width = 1pt] (y2);
    \plate [inner sep=.2cm,xshift=.02cm,yshift=.02cm] {plateUnroll1} {(x1)(y2)} {}; 
    \node at (4, -2.8) {\textbf{(c)}};
    
    \node[midstate] (x1) at (7, 1.5) {$X_1$} ;
    \node[midstate] (y1) at (9, 1.5) {$Y_1$} ;
    \node[midstate] (x2) at (7, -1.1) {$X_2$} ;
    \node[midstate] (y2) at (9, -1.1) {$Y_2$} ;
    \node[rectangle, fill = gray!20, thick, draw = black] (theta) at (7, 0.2) {$\theta$};
    \node[rectangle, fill = gray!20, thick, draw = black] (psi) at (9, 0.2) {$\boldsymbol{\psi}$};
    \path[->] (y1) edge[line width = 1pt] (x1);
    \path[->] (y2) edge[line width = 1pt] (x2);
    \path[->] (theta) edge[line width = 1pt] (x1);
    \path[->] (theta) edge[line width = 1pt] (x2);
    \path[->] (psi) edge[line width = 1pt] (y1);
    \path[->] (psi) edge[line width = 1pt] (y2);
    \plate [inner sep=.2cm,xshift=.02cm,yshift=.02cm] {plateUnroll1} {(x1)(y2)} {}; 
    \node at (8, -2.8) {\textbf{(d)}};
\end{tikzpicture}}
\caption{An illustration demonstrates different conditional independence relationships contained in i.i.d.\ process and ICM-generative process. (a): A causal graph generated under an i.i.d.\ process; (b): A causal graph generated under ICM-generative process; Unrolling the inner plate notation from (b), we visualize the process with two samples. Causal graphs $X \rightarrow Y$ and $Y \rightarrow X$ generated under an i.i.d.\ process share the same conditional independences $\{\emptyset\}$ and are thus observationally equivalent. (c) and (d) show the corresponding graphs under ICM-generative processes. (c) has $X_1 \ind Y_2 \mid X_2$ which does not hold in (d) and (d) has $X_1 \ind Y_2 \mid Y_1$ which does not hold in (c). One can thus differentiate the bivariate causal direction in ICM-generative processes.}
\label{fig:illus_e_equiv}
\end{figure}

\section{Preliminaries}

\subsection{The Causal Framework}
\label{sec:causal-framework}
A joint distribution $P(X_1, \ldots, X_N)$ over a set of variables $X_1, \ldots, X_N$ can be decomposed into simpler components in multiple ways. For example, by the chain rule of probability, one can factorize the joint as $P(X_1, \ldots, X_N) = \prod_{i=1}^N P(X_i \mid X_1, \ldots, X_{i-1})$. We say the joint distribution satisfies the \textit{Markov factorization} property with respect to a directed acyclic graph $\mathcal{G}$ if   
\begin{equation}
    P(X_1, \ldots, X_N) = \prod_{i=1}^N P(X_i \mid \textbf{PA}_i),
\label{eq:markovfactorization}
\end{equation}
where $\textbf{PA}_i$ are the direct parents of node $X_i$ in $\mathcal{G}$. While many factorisations can represent the same joint $P$, a specific one is called the \textit{causal Markov factorization}: in it, the factors  $P(X_i \mid \textbf{PA}_i)$ represent the causal mechanisms of the true underlying data generating process. One can use the causal factorization to predict effects of interventions, which we model mathematically by replacing the corresponding factor \citep{Pearl2009Causality:Inference.}.

Causal discovery aims at recovering the causal graph $\mathcal{G}$ and the corresponding causal Markov factorization from the joint distribution $P$. This can be done by matching the conditional independence structure implied by the graph $\mathcal{G}$ to those observable in the joint distribution $P$. To facilitate this matching, an elaborate graphical terminology has been developed, as detailed in Appendix \ref{sec:graphical-terminology}. Unfortunately, under the assumption that data is sampled i.i.d.\ from $P$, the true underlying $\mathcal{G}$ cannot be uniquely determined, only up to broad equivalence classes. The conditional independence structure of i.i.d.\ processes is not rich enough to facilitate identifiability of the causal structure $\mathcal{G}$.

\textbf{Independence of Causal Mechanisms, Causal and Anti-Causal Machine Learning} In addition to the study of Markov factorization, recent work \citep{Janzing2010CausalCondition} studies the behaviour of causal mechanisms and postulates the \textit{Independent Causal Mechanism (ICM)} principle, which states: \looseness=-1
\begin{center}
    Causal mechanisms are independent of each other in the sense that a change in one mechanism $P(X_i \mid \textbf{PA}_i)$ does not inform or influence any of the other mechanisms $P(X_j \mid \textbf{PA}_j)$, for $i \neq j$. 
\end{center}

The notion of invariant, independent and autonomous mechanisms have a long history in causality research: \cite{Haavelmo1944} and \cite{Aldrich1989} discuss the historical development of autonomous mechanisms in economics and \cite{Pearl2009Causality:Inference.} also argues that causal mechanisms are modular in nature. \cite{Scholkopf2012OnLearning} pointed out the implications of this principle when using machine learning techniques in  \textit{causal or anti-causal learning} settings, i.\,e., when the task is to predict an effect from a cause or a cause from an effect, respectively. The ICM principle implies that semi-supervised learning is only successful in the anti-causal direction, while the predictor can be robustly applied to new domains if learning is in the causal direction. While these observations seem intuitively true, it is difficult to ground their meaning in the langauge of probability or information. As we will see, these difficulties can be resolved once we consider non-i.i.d.\ data generating processes.

\subsection{Exchangeability}
\label{sec:exchangeability}
As we have seen, i.i.d.\ processes have a limitation that their conditional independence structure is not rich enough to support identifiability of the full causal graph. We thus turn our attention to a richer class of processes, exchangeable sequences.
\begin{definition}[Exchangeable sequence]
\label{def:exchangeable_sequence}
A finite sequence of random variables $X_1, X_2, \ldots, X_N$ is \textbf{exchangeable}, if for any permutation $\pi$ of its indices $\{1, \ldots, N\}$:
\begin{equation}
\label{eq:univariate_exchangeability}
P(X_{\pi(1)}, \ldots, X_{\pi(N)}) = P(X_1, \ldots, X_N)
\end{equation}
An \textbf{infinite exchangeable} sequence is a sequence where for any $N \in \mathbb{N}$, its finite sequence of length $N$ is exchangeable. 
\end{definition}

Exchangeability is a notion of symmetry. Definition \ref{def:exchangeable_sequence} informally states the order of observations does not matter. Recall a finite sequence of random variables is \textit{independent and identically distributed} if its joint distribution satisfies $P(X_1, \ldots, X_N) = \prod_{i=1}^N P(X_i)$. Of course, such an i.i.d.\ sequence is automatically exchangeable but not all exchangeable sequences are i.i.d. To clarify the connection between exchangeable and i.i.d.\ sequences, recall the de Finetti theorem:
\begin{theorem}[De Finetti's representation theorem \citep{deFinetti1931}]
\label{thm:de-finetti}
Let $(X_n)_{n \in \mathbb{N}}$ be an infinite sequence of binary\footnote{De Finetti's representation theorem has been extended to categorical and continuous variables \citep{Klenke2008ProbabilityCourse}.} random variables. The sequence is exchangeable if and only if there exists a random variable $\theta \in [0, 1]$ such that $X_1, X_2, ...$ are conditionally independent and identically distributed given $\theta$, with a probability measure $\mu$ on $\theta$. Mathematically speaking, given any sequence $(\mathbf{x}_1,.., \mathbf{x}_N) \in \{0, 1\}^N$:
\begin{equation}
\label{eq:definetti}
P(\mathbf{x}_1, \ldots, \mathbf{x}_N) = \int \prod_{i=1}^N p(\mathbf{x}_i \mid \theta) d\mu(\theta)
\end{equation}
\end{theorem}
Informally, the theorem states that an exchangeable sequence can always be represented as a mixture of i.i.d.\ sequences. 
De Finetti's representation theorem has important consequences for Bayesian inference. Bayesian statistics, unlike frequentist, takes the view that the parameter is a latent variable, instead of an unknown point estimate. Bayes' theorem estimates the parameter via calculating posterior density $p(\theta|\mathbf{x}_1,.., \mathbf{x}_N)$. De Finetti's representation theorem demonstrates \citep{o2009exchangeability} that rather than metaphysical belief about the true model, it is due to our judgement that the observations are exchangeable that underlies our standard use of Bayesian modelling involving observations are i.i.d.\ conditioned on some unknown latent variable.


\section{Causal de Finetti Theorems \label{sec:causal_de_finetti_bivariate}
}

Just as de Finetti justifies Bayesian modelling, we will introduce causal de Finetti theorems that justify causal modelling via  probability theory. 
We first motivate our study of exchangeable sequences by noting that they have a richer conditional dependence structure than i.i.d.\ sequences. We illustrate what this means concretely in the simplest possible case of a pair of variables.

Let $X$ denote a random variable and $x$ denote a random variable's particular realization. Let $[n]$ denotes the set of positive integers that are less than or equal to $n$, i.e. $[n] = \{1, \ldots, n\}$. 
\begin{definition}[Exchangeable pairs]\label{def:exchangeable_pair} An infinite sequence of random variable pairs $(X_n, Y_n)_{n \in \mathbb{N}}$ is exchangeable if for any permutation $\pi$ and for any finite number $N$, it satisfies
    \begin{equation}
    \label{eq:exchangeable_variable_pairs}
    P(X_{\pi(1)}, Y_{\pi(1)}, \ldots , X_{\pi(N)}, Y_{\pi(N)}) = P(X_1, Y_1, \ldots, X_N, Y_N)
    \end{equation}
\end{definition}

In an i.i.d.\ sequence over pairs, the only non-trivial independence is $X_i \ind Y_i$. Since the distribution is identical, it either holds for all $i$ or does not hold for all $i$. In an exchangeable sequence of pairs, one can consider other non-trivial independence, for example, $Y_i \ind X_j \mid X_i$. This conditional independence relationship trivially holds in i.i.d.\ sequences, but it may or may not hold in exchangeable sequences. Therefore, one can hope its absence or presence carries some useful information about some underlying causal structure. 
Here, we present the causal de Finetti theorems, which illustrate the type of causal structure this conditional independence relationship implies.

\begin{theorem}[Causal de Finetti -- bivariate]
\label{thm:causaldf-bivariate}
Let $\{(X_n, Y_n)\}_{n \in \mathbb{N}}$ be an infinite sequence of binary random variable pairs. The sequence is:
\begin{enumerate}
    \item infinitely exchangeable, and satisfies
    \item $\forall n \in \mathbb{N}$: $Y_{[n]} \ind X_{n+1} \mid X_{[n]}$ 
\end{enumerate}
if and only if there exist two random variables $\theta \in [0 ,1]$ and $\boldsymbol{\psi} \in [0, 1]^2$ with probability measures $\mu, \nu$ such that the joint probability can be represented as 
\begin{equation}
\label{causaldf}
 P(x_1, y_1, \ldots, x_N, y_N) = \int \prod_{n=1}^N p(y_n \mid x_n, \boldsymbol{\psi}) p(x_n \mid \theta) d\mu(\theta) d\nu(\boldsymbol{\psi})
\end{equation}
\end{theorem}

Informally, the theorem states that an exchangeable sequence of random variable pairs satisfying an additional set of conditional independence properties, can always be represented as a mixture of i.i.d.\ sequences which all share the same underlying Markov factorization structure, and thus, an \emph{invariant causal structure}. In fact, such exchangeable data can be interpreted as data generated under the ICM assumption. Recall the independent causal mechanism is loosely denoted as "$P_{\text{effect} \mid \text{cause}} \ind P_{\text{cause}}$". ICM assumption can be modelled by Eq.~\ref{causaldf} in the sense that causal mechanisms are characterized by latent variables: mechanisms do not influence each other if one can separately manipulate each latent variable controlling different mechanisms; further, latent variables governing each mechanism are statistically independent with each other, supporting mechanisms do not inform one another. We thus call the generative process in Eq.~\ref{causaldf} as ICM-generative process. 
 Causal de Finetti, just as how de Finetti substantiates Bayesian modelling, 
demonstrates that rather than metaphysical belief about independent causal mechanisms, it is due to our judgement that observations are exchangeable and the sufficiency to predict the target variable $Y$ with the corresponding $X$ values irrespective of other $X$ observations underlie our standard use of causal modelling. We thus substantiate ICM by detailing the statistical assumptions one implicitly make when assuming ICM. 

\begin{wrapfigure}[8]{r}{0.2\textwidth}
    \vspace{-0.3cm}
    \scalebox{0.8}{
    \begin{tikzpicture}
        \node (x1) at (0,0) [label=above:Patient$_1$,point];
        \node (y1) at (2,0) [label = above:Diagnosis$_1$, point];
        \node (theta) [below = 1cm of x1] [label = left:\rotatebox{90}{Hospital}, point];
        \node (psi) [below = 1cm of y1] [label = left:\rotatebox{90}{Doctor}, point];
        \node (x2)  [below = 1cm of theta] [label=below:Patient$_2$,point];
        \node (y2)  [below = 1cm of psi] [label = below:Diagnosis$_2$, point];
        \path[->] (x1) edge  (y1);
        \path[->] (x2) edge  (y2);
        \path[->] (theta) edge (x1);
        \path[->] (psi) edge (y1);
        \path[->] (theta) edge (x2);
        \path[->] (psi) edge (y2);  
    \end{tikzpicture}}
\end{wrapfigure}

As an example, consider the causal graph on the right. Imagine in a hospital there are two patients. A patient's symptom is the cause of a doctor's diagnosis. Suppose we are interested to predict a patient's diagnosis given her symptom. The conditional independence says knowing another patient's symptom will not help us to predict the diagnosis of this patient if we know this patient's symptoms already. The conditional independence thus formulated the intuition behind causal and anti-causal problem in the language of probability: the distribution of the cause, other patients' symptoms in this case, will not help prediction about the effect given cause, i.e.,  one patient's diagnosis given his own symptoms.

\textbf{Causal de Finetti vs.\ de Finetti} 
To see the difference between causal de Finetti and de Finetti's representation theorem, we observe that a direct application of de Finetti theorem on exchangeable data that may or may not contain causal information results in a factorization as: 
\begin{equation}
     P(x_1, y_1, \ldots, x_N, y_N) = \int \prod_{n=1}^N p(y_n, x_n \mid \theta) d\mu(\theta)
\end{equation}
As observations are conditionally i.i.d.\ given latent variable $\theta$, learning $\theta$ is thus sufficient to achieve maximum prediction power. This finding corroborates empirical results in the machine learning community, where training often produces an entangled representation that achieves strong prediction accuracy. However, with the fast development of powerful machine learning applications \citep{brown2020language}, both deep learning and causal communities advocate the need for disentangled representations \citep{Scholkopf2021TowardsLearning, bengio2013representation, locatello2019challenging}, which offer greater control, interpretability, and generalization capabilities. Causal de Finetti shows that, in fact, given exchangeable data satisfying the causal and anti-causal learning phenomenon formulated in conditional independences, there are statistically independent latent variables controlling each causal mechanism. It shows one can achieve both maximum prediction power and disentangled representations. Fig. ~\ref{fig:de-finetti-representation} illustrates a visualization of the differences between de Finetti and causal de Finetti theorems. 

We next illustrate causal de Finetti in the general multivariate form (see Appendix \ref{sec:proof_causal_de_finetti} for proof):
\begin{definition}[Exchangeable arrays] An array of size $d$ contains variables $(X_{1;n}, \ldots, X_{d;n})$ where $X_{d;n}$ denotes the $d$-th random variable in $n$-th array. Such an array is denoted as $\mathbf{X}_{:;n}$. A finite sequence of size $d$ arrays is \textbf{exchangeable}, if 
\begin{equation}
    P(\mathbf{X}_{:;\pi(1)}, \ldots, \mathbf{X}_{:;\pi(N)}) = P(\mathbf{X}_{:; 1}, \ldots, \mathbf{X}_{:;N})
\end{equation}
\end{definition}
\begin{theorem}[Causal de Finetti -- multivariate]
\label{thm:causaldf-multivariate}
Let $\{(X_{1;n}, X_{2;n}, \ldots X_{d;n})\}_{n \in \mathbb{N}}$ be an infinite sequence of $d$-tuple binary random variables. The sequence is
\begin{enumerate}
    \item infinitely exchangeable, and
    \item if there exists a DAG $\mathcal{G}$ such that $\forall i \in [d], \forall n \in \mathbb{N}$:
    $$X_{i;[n]} \ind \overline{\textbf{ND}}_{i;[n]}, \textbf{ND}_{i;n+1} | \textbf{PA}_{i;[n]}$$ where $\textbf{PA}_i$ selects parents of node $i$ and $\textbf{ND}_i$ selects non-descendants of node $i$ in $\mathcal{G}$. $\overline{\textbf{ND}}_i$ denotes the set of non-descendants of node $i$ excluding its own parents.
\end{enumerate}
if and only if there exist $d$ random variables where $\boldsymbol{\theta_i} \in [0, 1]^{2^{|\textbf{PA}_i|}}$ with suitable probability measures $\{\nu_i\}$ such that the joint probability can be written as
\begin{equation}
    P(\mathbf{x}_{:, 1:N}) = \int \int \prod_{n=1}^N \prod_{i=1}^d p(x_{i;n} \mid \textbf{pa}_{i;n}, \boldsymbol{\theta_i}) d\nu_1(\boldsymbol{\theta_1}) \dots d\nu_d(\boldsymbol{\theta_d}),
\label{eq:icmfactorization}
\end{equation}
where $\mathbf{x}_{:, 1:N} := \{(x_{1;n}, \ldots, x_{d;n})\}_{n=1}^N$.
\end{theorem}

Theorem \ref{thm:causaldf-bivariate} is a special case of Theorem \ref{thm:causaldf-multivariate} when $d=2$. Informally, it states that an exchangeable sequence of size $d$ random arrays satisfying an additional set of conditional independence properties with respect to a DAG, can always be represented as a mixture of i.i.d.\ sequences which all share the same underlying Markov factorization structure as the corresponding DAG. 
The set of conditional independence properties in condition $2$ can be interpreted via its decomposition:
$$X_{i;[n]} \ind \overline{\textbf{ND}}_{i;[n]} \mid \textbf{PA}_{i;[n]}$$
    This shows that the direct parents of one node form a Markov blanket for its other non-descendant nodes. In other words, to infer about the variable of interest,  it is sufficient to know the variable's direct parents irrespective of other non-descendants. 
$$X_{i;[n]} \ind \textbf{ND}_{i;n+1} \mid \textbf{PA}_{i;[n]}$$
This conditional independence is another example that exchangeable data has richer conditional independence structure. It trivially holds in i.i.d.\ sequences but may or may not hold in exchangeable data. Informally, it states that to infer the variable of interest, it is sufficient to know its corresponding parents irrespective of non-descendants in other observations. This formulated the intuition behind the causal and anti-causal problem in the language of probability: the distribution of the variables in the causal direction, in this case, non-descendants in other observations will not help prediction if covariates contain a complete set of the corresponding direct causal parents. 

\begin{wrapfigure}[8]{r}{0.3\textwidth}

    \scalebox{0.9}{
    \begin{tikzpicture}
        \node (x1) at (0,0) [label=above:Campaign$_1$,point];
        \node (y1) at (1.5,0) [label = above:Apply$_1$, point];
        \node (z1) at (3,0) [label = above:Admission$_1$, point];
        \node (theta) [below = 0.6cm of x1] [label = left:\rotatebox{90}{School}, point];
        \node (psi) [below = .6cm of y1] [label = left:\rotatebox{90}{Student}, point];
        \node (eta) [below = .6cm of z1] [label = left:\rotatebox{90}{Uni}, point];
        \node (x2)  [below = 0.6cm of theta] [label=below:Campaign$_2$,point];
        \node (y2)  [below = 0.6cm of psi] [label = below:Apply$_2$, point];
        \node (z2)  [below = 0.6cm of eta] [label = below:Admission$_2$, point];
        \path[->] (x1) edge  (y1);
        \path[->] (y1) edge (z1);
        \path[->] (x2) edge  (y2);
        \path[->] (y2) edge (z2);
        \path[->] (theta) edge (x1);
        \path[->] (psi) edge (y1);
        \path[->] (eta) edge (z1);
        \path[->] (theta) edge (x2);
        \path[->] (psi) edge (y2);
        \path[->] (eta) edge (z2);
        
    \end{tikzpicture}}
    \end{wrapfigure} 
To illustrate how Theorem \ref{thm:causaldf-multivariate} also justifies the ICM in action,  consider the causal graph to the right. Imagine high schools host campaigns to advertise university opportunities and encourage students to apply. Students' decision to apply is the cause of their university admission outcomes. Suppose we are interested in understanding how influential school campaigns are on students' decision to apply, i.e., the mechanism of "Apply $\mid$ Campaign". We expect that knowing a particular student's decision to apply after attending a school campaign will not be influenced by the other school campaigns the student did not attend; instead, knowing other students' decisions to apply after their attendance in campaigns will help the prediction of this particular student's decision to apply. Similarly, knowing other students' university admission outcomes will also be helpful. This is because more students decide to apply implies the effectiveness of school campaigns, and more university admission acceptance implies more students decide to apply. Causal de Finetti says given such assumptions and exchangeable data, it naturally holds that there exist latent variables, represented by high school, student and university, and they are independent. \looseness=-1

\textbf{Extension beyond binary} Above theorems are presented in its bivariate and multivariate forms for binary variables. In general, it is easy to extend the results to categorical variables\footnote{\citep{Barrett2009TheSpaces} provides an alternative proof of conditional de Finetti in quantum theory for categorical variables.}. Just as the progression of the proofs for de Finetti's theorem, we hypothesize causal de Finetti holds true for continuous variables. Here we state the theorem in its multivariate form for categorical variables.
\begin{theorem}[Causal de Finetti -- multivariate and categorical]
Consider an infinite sequence of size-$d$ random arrays $\{(X_{1;n}, X_{2;n}, \ldots X_{d;n})\}_{n \in \mathbb{N}}$, where each variable $X_{i;n}$ takes values in $\{1, \ldots, k_i\}$. The sequence is:
\begin{enumerate}
    \item infinitely exchangeable, and
    \item if there exists a DAG $\mathcal{G}$ such that $\forall i \in [d], \forall n \in \mathbb{N}$:
    $$X_{i;[n]} \ind \overline{\textbf{ND}}_{i;[n]}, \textbf{ND}_{i;n+1} | \textbf{PA}_{i;[n]}$$ where $\textbf{PA}_i$ selects parents of node $i$ and $\textbf{ND}_i$ selects non-descendants of node $i$ in $\mathcal{G}$. $\overline{\textbf{ND}}_i$ denotes the set of non-descendants of node $i$ excluding its own parents.
\end{enumerate}
if and only if there exist $d$ random variables where $\boldsymbol{\theta_i} \in [0,1]^{k_i \times \prod_{X_j \in \textbf{PA}_i} k_j } $ and every column of $\boldsymbol{\theta}_i$ sum to $1$ with suitable probability measures $\{\nu_i\}$ such that the joint probability can be written as
\begin{equation}
    P(\mathbf{x}_{:, 1:N}) = \int \int \prod_{n=1}^N \prod_{i=1}^d p(x_{i;n}|\textbf{pa}_{i;n}, \boldsymbol{\theta_i}) d\nu_1(\boldsymbol{\theta_1}) \dots d\nu_d(\boldsymbol{\theta_d}),
\end{equation}
where $\mathbf{x}_{:, 1:N} := \{(x_{1;n}, \ldots, x_{d;n})\}_{n=1}^N$.
\end{theorem}

\section{Identifiability Result}
\label{sec:identifiability}

The causal de Finetti theorems show that exchangeable processes can be represented as ICM generative process when certain CI statements hold. However, such a representation is only really useful for causal discovery if it is unique - in other words we would like if only one such decomposition were possible for any given exchangeable process. This property is called \emph{identifiability}. In this section we study the identifiability of ICM generative process. We start by introducing graphical terminology. 

\begin{definition}[Acyclic directed mixed graph (ADMG) \citep{Richardson2003}]
\label{def:admg}
    An acyclic directed mixed graph can contain two types of edges: directed '$\rightarrow$' or bi-directed '$\leftrightarrow$'. When an ADMG does not contain any bi-directed edge, it becomes a directed acyclic graph (DAG). 
\end{definition}

\begin{definition}[$\mathcal{I}$-map \citep{koller2009probabilistic}]
    Given $P$ is a distribution, $\mathcal{I}(P)$ denotes the set of conditional independence relationships of the form $X \ind Y \mid Z$ that hold in $P$. Given $\mathcal{G}$ be a ADMG, $\mathcal{I}(\mathcal{G})$ denotes the set of conditional independence assumptions encoded in $\mathcal{G}$ which can be directly read-off via m-separation \citep{zhang2012transformational}. When $\mathcal{G}$ is a DAG, $\mathcal{I}(\mathcal{G})$ can be directly read-off via d-separation \citep{pearl1988probabilistic}. 
\end{definition}

\begin{definition}[Global Markov property and faithfulness \citep{zhang2012transformational}]
Given a ADMG $\mathcal{G}$ and a joint distribution $P$, the distribution satisfies global Markov property with respect to $\mathcal{G}$ if $\mathcal{I}(\mathcal{G}) \subseteq \mathcal{I}(P)$. Sometimes it is also called $P$ is \textbf{Markovian} with respect to $\mathcal{G}$. We say $P$ is \textbf{faithful} to $\mathcal{G}$ if $\mathcal{I}(P) \subseteq \mathcal{I}(\mathcal{G})$.
\end{definition}

Definition \ref{def:icm-operator} defines the mapping from causal graphs generated under an i.i.d.\ process to those generated under an exchangeable process. The latent variables in causal de Finetti theorems can be represented by bi-directed edges in Definition \ref{def:admg}. 

\begin{definition}[ICM operator on a DAG]
\label{def:icm-operator}
   Let $U$ be the space of all DAGs whose nodes represent $X_1, \ldots, X_d$. Let $V$ be the space of ADMGs whose nodes represent $\{(X_{i;n})\}$, where $i \in [d], n \in \mathbb{N}$. A mapping $F$ from $U$ to $V$ is an ICM operator if $F(\mathcal{G})$ satisfies:
\begin{itemize}
    \item $F(\mathcal{G})$ restricted to the subset of vertices $\{X_{1; n}, \ldots, X_{d; n}\}$ is a DAG $\mathcal{G}$, for any $n \in \mathbb{N}$,
    \item $X_{i;n} \leftrightarrow X_{i;m}$ whenever $n \neq m$ for all $i \in [d]$
    \item there are no other edges other than stated above
\end{itemize}
We denote the resulting ADMG as $ICM(\mathcal{G})$. Let $\textbf{PA}_{i;n}^\mathcal{G}$ denote the parents of $X_{i;n}$ in $ICM(\mathcal{G})$ and similarly for $\textbf{ND}_{i;n}^\mathcal{G}$ for corresponding non-descendants.
\end{definition}

\begin{theorem}[Identifiability Theorem]
\label{thm:identifiability}
Consider the set of distributions that are both Markovian and faithful to $ICM(\mathcal{G})$, i.e., $\mathcal{E}(\mathcal{G}) := \{P: \mathcal{I}(P) = \mathcal{I}(\text{ICM}(\mathcal{G}))\}$.
Then, 
\begin{equation}
    \mathcal{E}(\mathcal{G}_1) = \mathcal{E}(\mathcal{G}_2) \  \text{if and only if} \  \mathcal{G}_1 = \mathcal{G}_2
\end{equation}
\end{theorem}
The set of graphs $\{\text{ICM}(\mathcal{G})\}$ characterizes the set of causal graphs for data sampled from ICM-generative processes. Given any two causal structures $\mathcal{G}_1$ and $\mathcal{G}_2$ underlying data sampled from ICM-generative processes, we say they are Markov equivalent if $\mathcal{E}(\mathcal{G}_1) = \mathcal{E}(\mathcal{G}_2)$. Theorem \ref{thm:identifiability} states given a distribution $P$ that is Markovian and faithful to $ICM(\mathcal{G})$, one can identify its unique causal structure as each graph induces unique conditional independences. 
See Appendix \ref{sec:proof_identifiability} for proof. 

\textbf{Connection to i.\,i.\,d.} Causal de Finetti theorems though stated formally under exchangeable process, it automatically holds for data generated under an i.i.d.\ process. When observing i.i.d.\ data, the measures $\nu_i$ in \eqref{eq:icmfactorization} are Dirac measures, and the de Finetti parameters $\{\theta_i\}$ are deterministic, i.e., fixed across multiple samples generated from the process. The identifiability theorem stated here, however, excludes distributions generated by marginally i.i.d.\ processes. It requires $P$ to be faithful to $ICM(\mathcal{G})$. If any one of the marginal distributions of $P$ can collapse to an i.i.d.\ process, i.e., there exists an index $d$ such that $P(X_{d;1}, \ldots, X_{d;N}) = \prod_n P(X_{d;n})$, then $P$ is not faithful to $ICM(\mathcal{G})$ since it contains extra conditional independence relationships. Fig.~\ref{fig:illus_e_equiv} illustrates that compared to i.\,i.\,d process, ICM-generative processes enable unique causal structure identification.

\section{Causal Structure Learning in Multi-environment Data} 
\label{sec:multi-env-data}

We established in Thm.\ \ref{thm:identifiability} that causal structure is identifiable in ICM generative models by testing for CI relationships in exchangeable data. For example, if  $Y_i \ind X_j \mid X_i$ holds for an exchangeable pair $(X_n, Y_n)$, we conclude that $X \rightarrow Y$, i.\,e.\ $X$ causes $Y$. But how exactly does one test for this in data?

To test if a CI statement holds between a set of random variables, one typically requires multiple samples, that is i.i.d.\ copies of the variables in question. Similarly, to apply our identification results in practice, we need access to multiple i.i.d.\ copies of the exchangeable pair $(X_n, Y_n), n\in \mathbb{N}$ (see Def.\ \ref{def:exchangeable_pair}). Each copy of $\{(X_n, Y_n)\}_{n\in \mathbb{N}}$ gives us a whole dataset containing a sequence of individual pairs, thus, we need multiple independent datasets to test for the CI condition. This requirement for multiple datasets connects our work to grouped or multi-environment data.

In the causal literature, grouped data refers to data available from multiple environments, each producing (conditionally) i.\,i.\,d observations from a different distribution, which are related through some \emph{invariant causal structure} shared by all environments. Grouped data underlies a wide range of causal discovery approaches \citep{Peters2016CausalIntervals, Tian2001CausalChanges, Heinze-Deml2018InvariantModels, Rojas-Carulla2018InvariantModels,huang2020causal,Arjovsky2019InvariantMinimization}. We can interpret multi-environment data through the lens of exchangeability as follow: In each environment $e \in \mathcal{E}$, we observe exchangeable samples $\mathbf{X}^e_{:;1:N_e}$ = \{$(X^e_{1;n}, \ldots,  X^e_{d;n})\}_{n=1\ldots N_e}$, where $X_{d;n}^e$ denotes the $d$-th random variable in $n$-th sample in environment $e$ and $N_e$ is the number of samples from environment $e$. Data across enviroments are independent and identically distributed in the sense that the distribution of $\mathbf{X}^e_{:;1:N}$ and $\mathbf{X}^{e'}_{:;1:N}$ is identical for all $N<\min(N_e, N_{e'})$. Each environment thus provides a finite marginal of an i.\,i.\.d copy of the same exchangeable process, just as we needed for testing CI. Alternatively, one can also interpret environments as samples from latent variables, i.e. ($\theta^e, \psi^e$) i.i.d.\ drawn from $p(\theta), p(\psi)$ characterizes environment $e$.  

Next, we propose the \textit{Causal-de-Finetti} algorithm, which guarantees to recover the DAG given multi-environment data consistent with ICM. In particular, the algorithm utilizes two samples per environment and a sufficiently large number of independent environments to enable identification. 

\textbf{Notation}: As every sample in all environments shares the same causal structure, we sometimes abbreviate the variable $X_{i;n}^e$ to $X_{i;n}$ or $X_i$. The results proved under abbreviated indices mean the abbreviated indices could take any values and the result remains the same. Let $S_n$ denotes the set containing nodes belong to $n-$th rank in a DAG $\mathcal{G}$'s topological ordering. (See details in Appendix \ref{sec:appendix_algorithm}) \looseness=-1

\begin{restatable}{lemma}{leafnodeidentification}
\label{lemma:leaf-node-identification}
  A node $X_{i;n} \in S_1$ if and only if for every $m \neq n$ and $j \neq i$, $X_{i;n} \ind X_{j;m} ~\vert~ \{X_{k;n}\}_{k\neq i}$.
\end{restatable}

\begin{restatable}{lemma}{edgeidentification}
\label{lemma:edgeidentification}
Let node $X_i \in S_n$ and $X_j \in S_m$ where $m < n$. Set $k:= n-m$. There does not exist a directed edge from $X_i$ to $X_j$ if and only if when $k = 1$,  $X_i \ind X_j \mid S_{>n}$; and when $k > 1$: $X_i \ind X_j \mid Z$, where $Z = S_{>n} \cup (\textbf{PA}_j \cap S_{<n}) \cup (S_n \backslash X_i)$.
\end{restatable}

Lemma \ref{lemma:leaf-node-identification} states the necessary and sufficient conditions to identify leaf nodes. Lemma \ref{lemma:edgeidentification}, intuitively says, to decide whether $X_i$ and $X_j$ have a direct edge, one should block all the potential non-directed paths. 
Step 1 of the algorithm is to iteratively identify and remove the current set of leaf nodes, and then search for the next set of leaf nodes until all nodes have been classified into their topological orders. Step 2 of the algorithm is to apply Lemma \ref{lemma:edgeidentification} to determine the existence of an edge between different topological orders. Algorithm \ref{alg:mvcausal_exchangeable} in Appendix \ref{sec:appendix_algorithm} details the exact procedure. 

\section{Experiments}
\label{sec:experiments}

We benchmark our method's performance against several state-of-the-art methods. As a measure of performance against methods for heterogeneous data, we compare against CD-NOD \citep{huang2020causal, Zhangetal17, huang2017behind}. As a measure of performance against methods designed for i.i.d. data, we compare with common causal structure learning algorithms, e.g. FCI \citep{spirtesmeekrichardson1995}, GES \citep{chickering2002optimal}, NOTEARS \citep{zheng2018dags}, DirectLinGAM \citep{shimizu2011directlingam} and PC algorithm \citep{spirtes2000causation}. Lastly, we compare with a random guess baseline. 

\textbf{Bivariate Causal Discovery} We generate multi-environment data as described in Section \ref{sec:multi-env-data}. Latent factors $\mathbf{N}$ were randomly generated with distinct and independent elements in each environment. Samples within each environment have the noise variables $\mathbf{\tilde{N}}$ generated via Laplace distributions conditioned on the latent factor. We observe bivariate data  $\mathbf{X} \in \mathbb{R}^2$ with $X_1$ and $X_2$ denotes the first and second entry of $\mathbf{X}$ and aim to uncover the causal direction between $X_1$ and $X_2$. Let $^e$ denote variables contained in environment $e$.
\begin{align*}
    & \mathbf{N^e} \sim \mathcal{U}[-1, 1] \\
    & \mathbf{\Tilde{N}^e} \sim \text{Laplace}(\mathbf{N}, 1) \\
    & \mathbf{X^e} = \mathbf{A^e}\mathbf{\tilde{N^e}} + \mathbf{B^e}\mathbf{\tilde{N}^{e \circ 2}}\mathds{1}_{\text{nonlinear}}(e)
\end{align*}
where $\circ 2$ denotes elementwise square operation. Specifically, $\mathbf{A}^e \in \mathbb{R}^{2 \times 2}$ is a randomly sampled triangular matrix and $\mathbf{B^e} = \mathbf{A^e} - \mathbf{I}$. We randomly sample bivariate structures, $X_1 \to X_2, X_2 \to X_1, X_1 \ind X_2$, by ensuring $\mathbf{A}$ is either a lower triangular, upper triangular or diagonal matrix. Our data further simulates a realistic situation, i.e., the causal structure remains invariant with changing functional relationships across environments. This is implemented by randomly sampling the existence of nonlinear dependence indicator $\mathds{1}_{\text{nonlinear}}(e)$ per environment. We perform three conditional independence tests with $\alpha = 0.05$ and output our estimate as the causal structure corresponding to the test with the highest $p$-value. We repeat the experiment for $100$ times and report the proportion of correct causal direction identified with varying numbers of environments.  Figure \ref{fig:experiment-results}(a) shows the proportion of correct bivariate causal direction detected as the number of environments $|\mathcal{E}|$ increases and the number of observations within each environment as fixed to be $2$. We observe Causal-de-Finetti algorithm outperforms all the other state-of-the-art methods and its accuracy converges close to $100\%$. This demonstrates its capability to handle datasets with limited samples per environment and changing functional relationships across environments. 

\begin{figure}%
    \captionsetup[subfigure]{oneside,margin={-1.5cm,0cm}}
    \centering
    \subfloat[Bivariate causal discovery]{{\includegraphics[width=6cm]{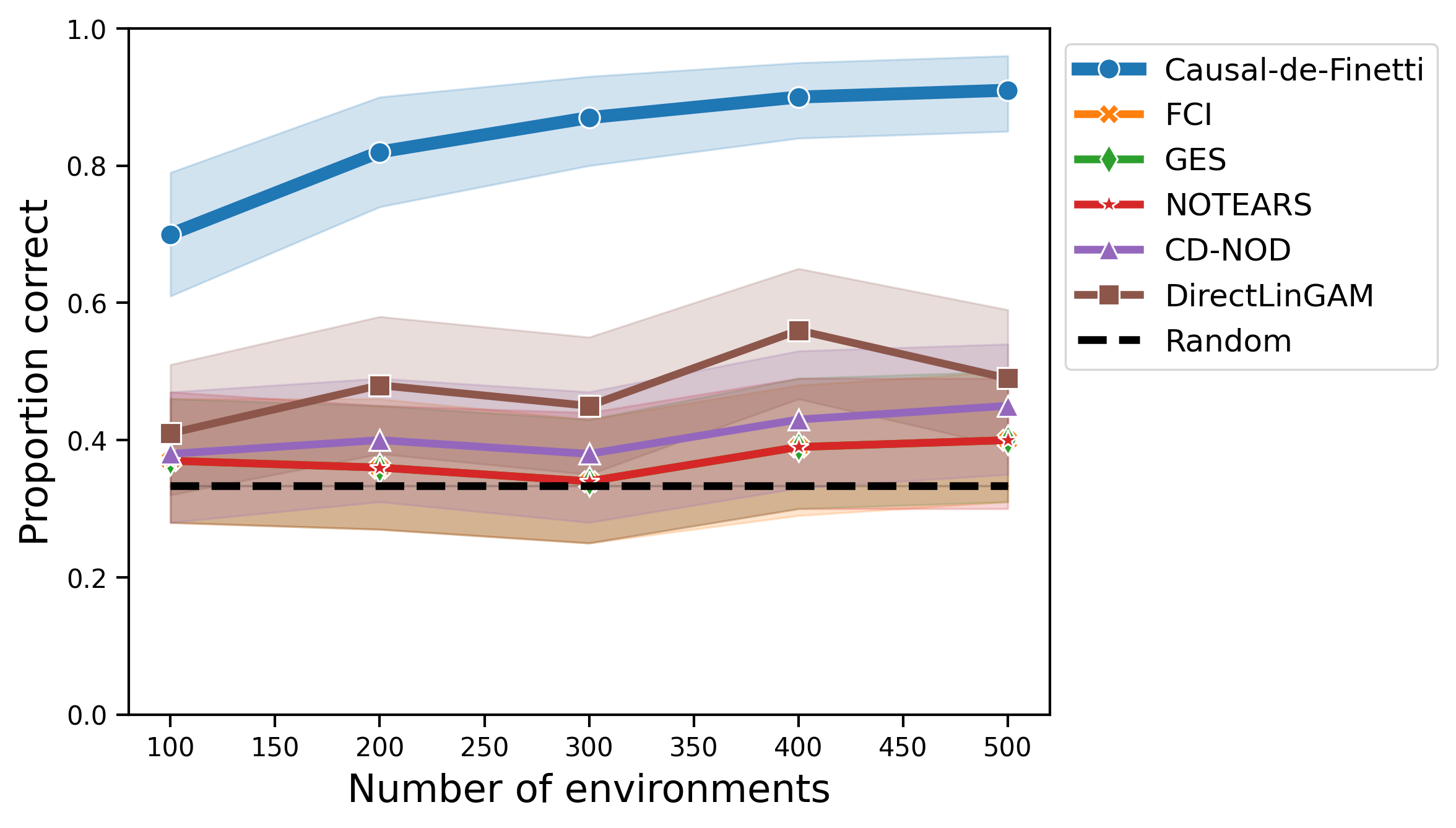} }}%
    \qquad
    \subfloat[Multivariate structure learning]{{\includegraphics[width=6cm]{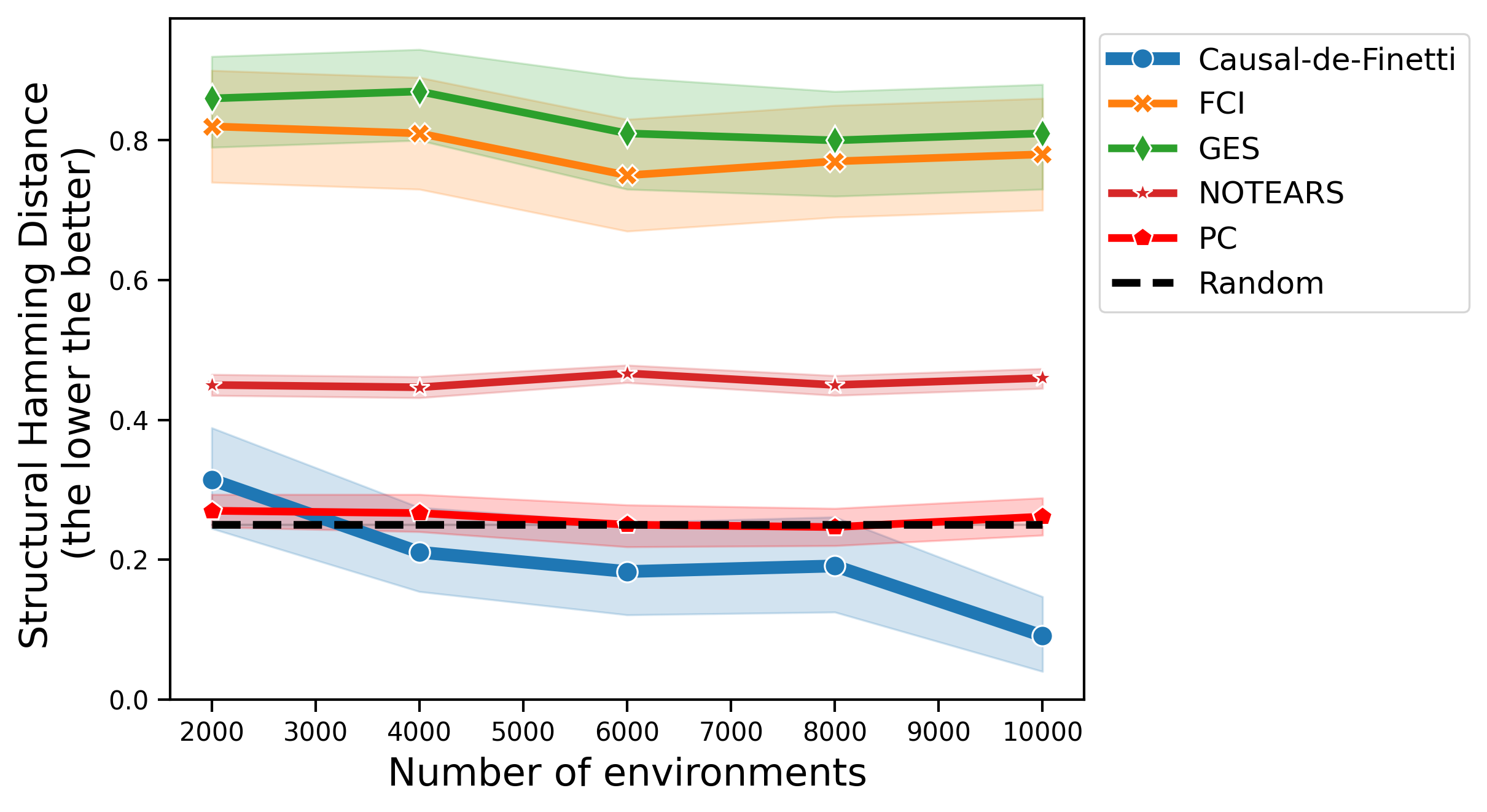} }}%
    \caption{Our method’s (“Causal-de-Finetti”) performance in identifying the correct underlying DAG, compared to the “CD-NOD”, “FCI”, “GES”, “NOTEARS”, "DirectLinGAM", "PC", "Random" baseline in bivariate and multivariate settings. Shown are the mean and $95\%$ confidence interval of the standard error of the mean for each method aggregated over $100$ experiments. "Causal-de-Finetti" identifies unique causal structures and is robust against changing functions across environments. \looseness=-1}
    \label{fig:experiment-results}%
\end{figure}

\textbf{Multivariate Causal Structure Discovery} We further test our algorithm's performance in identifying multivariate causal structure. We randomly generate causal graphs with $3$ variable nodes where each variable takes binary values. Figure \ref{fig:experiment-results}(b) shows the Structural Hamming Distance \citep{Tsamardinos2006} between true and estimated DAG averaged over $100$ experiments. We again observe Causal-de-Finetti outperforms all the other baselines, which validates our identifiability theorem and demonstrates algorithms designed for i.i.d.\ data does not work in our setting.

\newpage
\section{Discussion}
\label{sec:related-work}
\textbf{Causal exchangeability} 
\cite{dawid2021decision} introduces a decision-theoretic framework for causality and uses pre-treatment and post-treatment exchangeability as foundational assumptions on external data used to solve the decision problem. 
\cite{Jensen2020} studies object conditioning and show its probabilistic interpretations can be explained using exchangeability. Object conditioning, due to exchangeability, thus mitigates latent confounding and measurement errors for causal inference.
Our work provides a statistical understanding of ICM assumption: it is equivalent to assuming exchangeability and certain conditional independence conditions. \looseness=-1

\textbf{Causal structure learning} Within the study of i.i.d.\ data, it is well-known that one can only identify causal structures up to Markov equivalence classes \citep{pearl1988probabilistic}, and going beyond that is known to be impossible without further parametric constraints \citep{hoyer2008nonlinear, Shimizu2006ADiscovery}. A recent line of work considers a mixture of observational multi-environment data and interventional data to perform inference: \cite{Peters2016CausalIntervals, Rojas-Carulla2018InvariantModels, Arjovsky2019InvariantMinimization} use causal invariance property to discover stable predictors and \cite{huang2017behind, huang2020causal} estimate kernel mean embeddings of heterogenous data distributions to test the independence of causal mechanisms. \cite{monti2020causal} discovers bivariate causal direction by first recovering the underlying generating sources and then performing conditional independence tests on the recovered source factors. These algorithms on non-i.i.d.\ grouped data all demonstrate success, though it is unclear the connection between causal assumptions and probabilistic implications of grouped data. Our work observes that grouped data is akin to exchangeable sequences, containing richer conditional independence structures. In particular, ICM-generative processes with a sufficient number of environments allow unique causal structure identification.

\textbf{Relations to causality in time-series} The study of data generated from i.i.d.\ process, exchangeable process, and time-series can be seen as a progression to understand more structured data. Appendix \ref{sec:time-series} details the connections between ICM-generative processes and causality in time series \citep{runge2020discovering}.  


\textbf{Conclusion} We prove causal de Finetti theorems formalizing the independent causal mechanism assumption in data generating processes as concrete statistical conditions. We call the induced generative models ICM-generative processes. For  data sampled from ICM-generative processes, we show that one can identify unique causal structure. We build the connection between exchangeable and grouped data and justify the success of many methods leveraging ICM and algorithms in non-i.i.d.\ grouped data. Going beyond the i.i.d.\ assumption has been the bottleneck to applying machine learning to real-world situations. Rather than considering it a nuisance, our work shows an example of a theoretical advantage of exchangeable data in causal structure identification. 

\textbf{Acknowledgement} S.G.\ would like to acknowledge helpful discussions with Damon Wishick on understanding causal de Finetti theorem in multi-environment data and Thijs van Ommen for his constructive feedback in the reviewing process which corrected the original version of Lemma 2.

\bibliographystyle{abbrvnat}
\bibliography{main}

\newpage
\appendix

\section{Graphical Terminology}
\label{sec:graphical-terminology}

An arbitrary graph $\mathcal{G}$ consists of vertices $V$ and edges $E \subseteq V^2$ with $(v, v) \not \in E$ for any $v \in V$. Then $\mathcal{G} = (V, E)$ is a graph with $V:=\{1, \ldots, d\}$ and corresponding random variables $\{X_1, \ldots, X_d\}$. A variable $X_i$ is called a parent of $X_j$ if $(i, j) \in E$ and $(j, i) \not \in E$ and a child if $(j, i) \in E$ and $(i, j) \not \in E$. The set of parents of $X_j$ in $\mathcal{G}$ is denoted as $\textbf{PA}_i^\mathcal{G}$, and the set of its children by $\textbf{CH}_i^\mathcal{G}$. Whenever the graph $\mathcal{G}$ is obvious from the context, one can omit its specification in the above notations. Two variables $X_i$ and $X_j$ are adjacent if either $(i, j) \in E$ or $(j, i) \in E$. A pair of variables can be connected with a directed edge $X_i \to X_j$. If there does not exist a sequence of edges such that $X_i \rightarrow \dots \rightarrow X_i$ for all $i\in V$, then $\mathcal{G}$ is acyclic. 

A path in $\mathcal{G}$ is a sequence of (at least two) distinct vertices $i_1, \ldots, i_m$ such that there is an edge between $i_k$ and $i_{k+1}$ for all $k=1, \ldots, m-1$. If $i_k \to i_{k+1}$ for all $k$, then $X_{i_1}$ is an ancestor of $X_{i_m}$, and that $X_{i_m}$ is a descendant of $X_{i_1}$. The set of ancestors of $X_i$ is denoted as $\textbf{AN}_i^\mathcal{G}$ and $\textbf{DE}_i^\mathcal{G}$ denotes the set of descendants of $X_i$. All non-descendants of $X_i$, excluding itself, are denoted as $\textbf{ND}_i^\mathcal{G}$. In this work we use $\overline{\textbf{ND}}_i^\mathcal{G}$ to denote the set of non-descendants excluding its parents. 

Causal structure learning via performing conditional independence tests involves matching conditional independences contained in probability distributions with the conditional independence assumptions encoded in the graph. D-separation \citep{pearl1988probabilistic} provides a graphical criterion that characterizes the set of conditional independences in the graph. 

\begin{definition}[d-separation]
Given a directed acyclic graph $\mathcal{G}$, a path $p$ with vertices $i_1, \ldots, i_m$ is d-separated by a block of nodes $Z$ if and only if one of the two conditions holds: 
\begin{itemize}
    \item $p$ contains a chain $i_{k-1} \rightarrow i_k \rightarrow i_{k+1}$ or $i_{k-1} \leftarrow i_k \leftarrow i_{k+1}$, or a fork $i_{k-1} \leftarrow i_l \rightarrow i_{k+1}$ and $i_k \in Z$; 
\item $p$ contains a collider $i_{k-1} \rightarrow i_k \leftarrow i_{k+1}$ s.t. the middle node $i_k \not \in Z$ and none of its descendants is in $Z$. 
\end{itemize}
We then say $Z$ d-separates two disjoint subsets of vertices $X$ and $Y$ if it blocks every path from a node in $X$ to a node in $Y$, and write as $X \ind_\mathcal{G} Y | S$. 
\end{definition}
We refer the readers to \citep{Janzing2017ElementsAlgorithms} for more detailed graphical terminology.  

\section{Proof of Causal de Finetti} 
\label{sec:proof_causal_de_finetti}
Here we refer to Causal de Finetti as in its multivariate form, as bivariate is a subcase contained in multivariate form. We base our proof mostly on \citep{Kirsch2019AnTheorem}. 

\textbf{Preliminaries} For a probability measure $\mu$ on $\mathbb{R}^d$ we define the mixed moment by $m_\mathbf{a}(\mu) := \int \prod_{i=1}^d x_i^{a_i} d\mu(x_1, \ldots, x_d)$ whenever it exists (in the sense that $\int \prod_{i=1}^d |x_i|^{a_i} d\mu(x_1, \ldots, x_d) < \infty $ ). Below we will only deal with measures with compact support so that all moments exist and are finite. The following is a multivariate extension of method of moments:
\begin{theorem}[Multivariate method of moments]\label{thm:mv_method_moments}
Let $\mu_n(n \in \mathbb{N})$ be probability measures with support contained in a fixed interval $[a, b]^d$. If for all $\mathbf{u}$ the mixed moments $m_{\mathbf{u}}(\mu_n)$ converge to some $m_\mathbf{u}$ then the sequence $\mu_n$ converges weakly to a measure $\mu$ with moments $m_\mathbf{u}(\mu) = m_\mathbf{u}$ and with support contained in $[a, b]^d$. Further, if $\mu$ is a probability measure with support contained in $[a, b]^d$ and $\nu$ is a probability measure on $\mathbb{R}^d$ such that $m_\mathbf{u}(\mu) = m_\mathbf{u}(\nu)$ then $\mu = \nu$.
\end{theorem}
\begin{proof}[Proof of Theorem \ref{thm:mv_method_moments}]
The first statement follows directly from the first theorem in \citep{Haviland1936OnII}. The second statement can be shown using a similar argument as univariate case in \citep{Kirsch2019AnTheorem} with Weierstrass approximation theorem.
\end{proof}
\textbf{Notation} Let $X_{i;n}$ denotes $i$-th random variable in $n$-th sample. We write $\mathbf{X_n}:= (X_{1;n}, \ldots, X_{d;n})$ and  $X_{d;:} := (X_{d;1}, \dots, X_{d;N})$. Define $UP_{i;:} := (X_{i+1;:}, \ldots, X_{d;:})$, which contains all random variables that have higher variable index value than $i$, i.e. upstream of node $i$.
\begin{definition}[Topological Ordering]
A topological ordering of a DAG is a linear ordering of its nodes such that for every directed edge $X \rightarrow Y$, $X$ comes before $Y$ in the ordering.
We call the ordering is a reversed topological order if we reverse the topological ordering of a DAG.
\end{definition}

Without loss of generality, we reorder the variables according to reversed topological ordering, i.e. a node's parents will be placed after this node. Note a reversed topological ordering is not unique, but it must satisfy a node's descendants will come before itself. 
Then by Kolmogorov's chain rule, we can always write the joint probability distribution as
\begin{equation}
    P(\mathbf{X_1}, \dots, \mathbf{X_N}) = \prod_{i=1}^d \tilde{P}(X_{i;:}|UP_{i;:})
\label{eq:bayesfac_proof}
\end{equation}
Note $P$ and $\tilde{P}$ are not the same, for ease of notation, we use $P$ below in general. For each $X_{i;:}$, we want to show there exists a suitable probability measure $\nu_i$ such that we can write
$P(X_{i;:}|UP_{i;:}) = \int \prod_{n=1}^N p(X_{i;n} | PA_{i;n}, \boldsymbol{\theta_i}) d\nu_i(\boldsymbol{\theta_i})$. Then substitute back into Equation \ref{eq:bayesfac_proof} we will have Causal de Finetti.

\begin{theorem}[Causal Conditional de Finetti] \label{thm:icm_cond_definetti}
Let $\{(X_{i;n}, X_{i+1;n}, \dots X_{d;n})\}_{n \in \mathbb{N}}$ satisfies conditions 1) and 2) in Causal de Finetti. Then there exists a suitable probability measure $\nu$ such that the conditional probability can be written as
\begin{equation}
P(X_{i;:}|UP_{i;:}) = \int \prod_{n=1}^N p(X_{i;n} | PA_{i;n}, \boldsymbol{\theta}) d\nu(\boldsymbol{\theta})
\end{equation}
where $\boldsymbol{\theta}$ is a vector where its index represents a unique combination for $PA_i$ values. We can thus consider $ p(X_{i;n} | PA_{i;n}, \boldsymbol{\theta}) = \pi_{\theta_j}(X_{i;n})$ where $j$ is the index of $PA_{i;n}$ in all possible realizations of $PA_i$ and $\pi_p$ is a Bernoulli probability measure parameterized by $p \in [0, 1]$.  
\end{theorem}

\begin{lemma}
Let $\{(X_{i;n}, X_{i+1;n}, \dots X_{d;n})\}_{n \in \mathbb{N}}$ satisfies conditions 1) and 2) in Causal de Finetti. Then for every permutation $\pi$ of $\{1, 2, ..., N\}$:
\begin{equation}
\begin{split}
& \mathbb{P}(X_{i;1}, ..., X_{i;N}|UP_{i,1}, ..., UP_{i;N}) \\
    = &\mathbb{P}(X_{i;\pi(1)}, ..., X_{i;\pi(N)}|UP_{i;\pi(1)}, ..., UP_{i;\pi(N)})
\end{split}    
\end{equation}
\label{lemma:conditional_is_exchangeable}
\end{lemma} 

\begin{proof}[Proof of Lemma \ref{lemma:conditional_is_exchangeable}]
Since $\{(X_{i;n}, X_{i+1;n}, \dots X_{d;n})\}_{n \in \mathbb{N}}$ is exchangeable, it is clear by marginalizing $\{X_{i;n}\}_{n\in\mathbb{N}}$ from definition we have  $\{UP_{i;n}\}_{n\in\mathbb{N}}$ is also exchangeable. 
\begin{equation*}
\begin{split}
    &  \mathbb{P}(X_{i;1}, ..., X_{i;N}|UP_{i,1}, ..., UP_{i;N})\\
    = &\frac{\mathbb{P}(X_{i;1},UP_{i,1}, ..., X_{i;N}, UP_{i;N})}{\mathbb{P}(UP_{i,1}, ..., UP_{i;N})} \\
    = & \frac{\mathbb{P}(X_{i;\pi(1)}, UP_{i;\pi(1)}, ..., X_{i;\pi(N)}, UP_{i;\pi(N)})}{\mathbb{P}(UP_{i;\pi(1)} , ..., UP_{i;\pi(N)})} \\
    = &\mathbb{P}(X_{i;\pi(1)}, ..., X_{i;\pi(N)}|UP_{i;\pi(1)} , ..., UP_{i;\pi(N)})
\end{split}    
\end{equation*}

Note since we reorder the index of multivariate random variables according to reversed topological ordering, we have $PA_{i;n} \subseteq UP_{i;n}$, so given $UP_{i;n}$ we would know $PA_{i;n}$.
This lemma implies for every conditional distribution we can always choose a permutation such that we can group values with identical $PA_i$'s realizations together. For example, when $|PA_i| = 1$, then we can permute such that all observations with $PA_{i;n} = 0 $ comes first and observations with $PA_{i;n} = 1$ come second. Let's order all possible realizations of $PA_i$ into a list of length $K:=2^{|PA_i|}$ and index each realization. Then from observations we have $N_k$ pairs which have $PA_i$ takes values as the index $k$'s realization.  Here we assume we observe enough samples to see every realization of $PA_i$. This is possible because $K$ is finite.
Then we can rearrange such that
\begin{equation*}
\begin{split}
    &  \mathbb{P}(X_{i;1}, ..., X_{i;N}|UP_{i,1}, ..., UP_{i;N}) \\
    = &P(\{\{X^k_{i;n}\}_{n=1}^{N_k}\}_{k=1}^K|\{\{\UPink \}_{n=1}^{N_k}\}_{k=1}^K ) \\
\end{split}
\end{equation*}
where $X^k_{i;n}$ denotes that its parents $PA_{i;n}$ takes realizations the same as index $k$ indicates and $\UPink $ denotes that the random vector $UP_{i;n}$ contains $PA_i$ which takes realizations the same as index $k$ indicates.
\end{proof}

\begin{corollary} \label{corollary_icm}
For any $K$-tuple permutations $(\pi_1, \pi_2, \ldots, \pi_K)$ where $\pi_k$ permutes $\{1, ..., N_k\}$:
\begin{equation*}
    \begin{split}
       &P(\{\{X^k_{i;n}\}_{n=1}^{N_k}\}_{k=1}^K|\{\{UP^k_{i;n}\}_{n=1}^{N_k}\}_{k=1}^K )\\
       = &P(\{\{X^k_{i;\pi_k(n)}\}_{n=1}^{N_k}\}_{k=1}^K |\{\{UP^{k}_{i;\pi_k(n)}\}_{n=1}^{N_k}\}_{k=1}^K )
    \end{split}
\end{equation*}
\end{corollary}
\begin{proof}[Proof of Corollary \ref{corollary_icm}]
Follows directly from Lemma 1.
\end{proof}

\begin{lemma} \label{cond2equiv}
Recall condition 2) in Causal de Finetti states that $\forall i, \forall n \in \mathbb{N}$:
    $$X_{i;[n]} \ind \overline{ND}_{i;[n]}, ND_{i;n+1} | PA_{i;[n]}$$
By exchangeability, it is equivalent to $$X_{i;I} \ind \overline{ND}_{i;I}, ND_{i;m} | PA_{i;I}$$ where $I$ is any set and $m\not \in I$. 
\end{lemma}
\begin{lemma}
Let \vecXin  satisfies conditions 1) and 2) in Causal de Finetti. There exists $K$-infinitely exchangeable sequence $\{\{ \Xinkstar \}_{n\in \mathbb{N}}\}_{k = 1}^K$ such that for every $N_k \in \mathbb{N}, \forall k$ we have:
\begin{equation}
\begin{split}
    &P( \{\{ \Xinkstar\}_{n=1}^{N_k}\}_{k=1}^K) \\
    & = P( \{\{\Xink\}_{n=1}^{N_k}\}_{k=1}^K|\{\{ \UPink \}_{n=1}^{N_k}\}_{k=1}^K )
\end{split}
\end{equation}
where $k$ is the index for $PA_i$'s particular realization.
\label{lemma:existcondexchseq}
\end{lemma}

\begin{proof}[Proof of Lemma \ref{lemma:existcondexchseq}]
To show such sequence exists, we need to show it is well-defined and the inductive defining sequence is consistent. \\
To show it is well-defined: 
\begin{equation*}
\begin{split}
& P( \{\{\Xink\}_{n=1}^{N_k}\}_{k=1}^K|\{\{\UPink\}\}_{n=1}^{N_k}\}_{k=1}^K ) \\
&= P( \{\{\Xink\}_{n=1}^{N_k}\}_{k=1}^K|\{\{PA^k_{i;n}\}\}_{n=1}^{N_k}\}_{k=1}^K ) \\
&= P(\{\{ \Xinkstar\}_{n=1}^{N_k}\}_{k=1}^K)
\end{split}
\end{equation*}

Using Lemma \ref{cond2equiv}, let $I = \{\{(k;n)\}_{n=1}^{N_k}\}_{k=1}^K$ and we have $\overline{UP}_i \subseteq \overline{ND}_i$ since any particular reversed topological sort will place node $i$'s decendants before itself. Then condition 2) implies $X_{i;I} \ind \overline{UP}_{i;I}| PA_{i;I}$ by decomposition rule in conditional independence. Because index $k$ already characterizes the value of $PA_i$ so result follows by definition.\\
consistent: We write $\{\{\cdot\}_{n=1}^{N_k}\}_{k=1}^K$ as $\{\{\cdot\}\}$ for abbreviation. For any $k$, consider
\begin{align*}
     &P( \{\{ \Xinkstar\}\})\\
     =&P( \{\{\Xink \}\}|\{\{\UPink\}\}) \\
     =&P( \{\{\Xink\}\}|\{\{\UPink\}\}, UP_{i;N_k + 1}^k)\\
     =&\sum_{X^i_{k;N_k+1}} P(\{\{ \Xink \}\}, X^k_{i;N_k+1}|\{\{UP_{i;n}^k\}\}, UP^k_{i;N_k + 1})\\
     =&\sum_{X^{i, *}_{k;N_k+1} = 0}^1 P(\{\{\Xinkstar\}\}, X^{k, *}_{i;N_k+1})
\end{align*}
The first equality holds by well-defindedness. 
Let $I = \{\{(k;n)\}_{n=1}^{N_k}\}_{k=1}^K$ and $m = (k;N_k+1)$. Note $ \overline{UP}_i \subseteq \overline{ND}_i$. Lemma \ref{cond2equiv} implies the second equality holds. The third equality holds by marginal property of probability distribution. The fourth equality follow from well-definedness. Infinite exchangeability of $\{\Xinkstar\}_{n\in \mathbb{N}}, \forall k$  follows from Corollary \ref{corollary_icm}. 
\end{proof}

\begin{definition}[Causal Conditional de Finetti measure] \label{def:icm_definetti_measure}
Using the notation introduced in Lemma \ref{lemma:existcondexchseq}, we define a random vector $\boldsymbol{Q}$ where $Q_k = \frac{1}{N_k}\sum_{n=1}^{N_k} \Xinkstar$. Let the joint distribution of $\boldsymbol{Q}$ be $\nu_{N_1;\ldots;N_K}$ or in shorthand $\nu_{\boldsymbol{N}}$ where $\boldsymbol{N}:= [N_1, ..., N_K]$.  If $\nu_{\boldsymbol{N}}$ converges to a probability measure $\nu$ as $N_k \rightarrow \infty, \forall k$, we call $\nu$ the Causal conditional de Finetti measure.  
\end{definition}

\begin{lemma}\label{lemma:unique_idx_def}
        Let $\{A_i\}_{i \in \mathbb{N}}$ be an infinitely exchangeable random binary process. Given a list of indices $\{i_1, \ldots, i_n\}$, let us denote the number of unique elements with $\rho(i_1, \ldots, i_n)$. For every arbitrary list of indices, the following holds:
    \begin{equation}
       \mathbb{E}[A_{i_1}A_{i_2} \ldots A_{i_n}] = \mathbb{E}[A_{1}A_{2} \ldots A_{\rho(i_1, \ldots, i_n)}] 
    \end{equation}
\end{lemma}

\begin{proof}[Proof of Lemma \ref{lemma:unique_idx_def}]
For binary variables, we have $\mathbb{E}[A_i^l] = \mathbb{E}[A_i], \forall l \geq 1$. So the product from left hand side is a product of $\rho(i_1, \ldots, i_n)$ different $A_i$'s. Due to exchangeability,  we have right hand side. 
\end{proof}

\begin{lemma} \label{lemma:unique_idx_icm}
Let $\{\{A_{k;i}\}_{i \in \mathbb{N}}\}_{k=1}^K$  be $K$ infinitely joint exchangeable random binary processes. For every $K$ arbitrary list of indices $\{\{\mathbf{i_k}\}\}_{k=1}^K$, where $\mathbf{i_k} := (i_{k;1}, \ldots, i_{k;N_k})$ denotes the sequence of indices selected for $k$-th process and $N_k$ is the number of indices. the following holds:
    \begin{equation}
    \begin{split}
      & \mathbb{E}[\prod_{k=1}^K \prod_{i \in \mathbf{i_k}} A_{k;i}] \\
      = & \mathbb{E}[\prod_{k=1}^K \prod_{i=1}^{\rho(\mathbf{i_k})} A_{k;i}] 
      \end{split}
    \end{equation}
\end{lemma}

\begin{proof}
Since the two sequences are jointly exchangeable, for $K$ sets of indices, we can independently perform the argument in Lemma \ref{lemma:unique_idx_def}. Hence only the number of unique indices in each set matters, which is the same on both sides. 
\end{proof}

\begin{theorem}
\label{theorem:moment_expression}
If we allow $N_k \rightarrow \infty, \forall k$, the probability measure $\nu_{\boldsymbol{N}}$ converges to a probability distribution $\nu$. The measure $\nu$ has the following joint moments:
$$
m_{\boldsymbol{u}}(\nu) = \mathbb{E}[\prod_{k = 1}^K \prod_{n=1}^{u_k} \Xinkstar]
$$
\end{theorem}
\begin{proof}
We will first examine the mixed moments of $\nu_{\boldsymbol{N}}$. 

\begin{equation*}
\begin{split}
    &\lim_{\boldsymbol{N} \rightarrow \infty} m_{\boldsymbol{u}}(\nu_{\boldsymbol{N}}) \\
    &= \mathbb{E}\Big[ \lim_{\boldsymbol{N} \rightarrow \infty} \Big( \prod_{k=1}^K \frac{1}{N_k^{u_k}}\big(\sum_{n=1}^{N_k} \Xinkstar \big)^{u_k}\Big) \Big]\\
    &= \mathbb{E}\Big[\lim_{N_k \rightarrow \infty} \frac{1}{N_k^{u_k}} \big( \sum_{\mathbf{i_k}} \prod_{n \in \mathbf{i_k}} \Xinkstar \big)\\
    &\quad \quad \prod_{j \neq k} \lim_{N_j \rightarrow \infty} 
    \Big(\frac{1}{N_j^{u_j}}\big(\sum_{n=1}^{N_j} X^{i,*}_{j;n}\big)^{u_j}\Big) \Big] \\
\end{split}
\end{equation*}
The second equality is possible by Lebesgue dominated convergence theorem and we have $|Q_k| \leq 1 , \forall k, N_k$.  The third equality is due to the product of limits is the limit of products. 
Next for any $k$, we focus on understanding each individual limit. 
\begin{equation*}
    \begin{split}
        &\lim_{N_k \rightarrow \infty} \Big(\frac{1}{N_k^{u_k}} \big( \sum_{\mathbf{i_k}} \prod_{n \in \mathbf{i_k}} \Xinkstar \big)\Big) \\
        &= \lim_{N_k \rightarrow \infty}  \frac{1}{N_k^{u_k}} \big( \sum_{\mathbf{i_k}:\rho(\mathbf{i_k}) < u_k} \prod_{n \in \mathbf{i_k}} \Xinkstar \big) \\
        & + \lim_{N_k \rightarrow \infty}  \frac{1}{N_k^{u_k}} \big( \sum_{\mathbf{i_k}:\rho(\mathbf{i_k}) = u_k} \prod_{n \in \mathbf{i_k}} \Xinkstar \big) \\
    \end{split}
\end{equation*}
When $N_k \rightarrow \infty$, since $\Xinkstar$ are binary variable, the first term becomes:
\begin{equation}
\label{Eq:converge_icm}
    \begin{split}
        0 & \leq \frac{1}{N_k^{u_k}} \sum_{\mathbf{i_k}:\rho(\mathbf{i_k}) < u_k} \big(\prod_{n \in \mathbf{i_k}} \Xinkstar \big)\\
        &\leq \frac{1}{N_k^{u_k}} \sum_{\mathbf{i_k}:\rho(\mathbf{i_k}) < u_k} 1 
    \end{split}
\end{equation}

The number of possible tuples of indices $\mathbf{i_k}$ with $\rho(\mathbf{i_k}) < u_k$ is at most $(u_k-1)^{u_k} N_k^{u_k-1}$. Because we have $N_k^{u_k-1}$ possibilities to choose the possible candidates for $(i_{k;1}, \ldots, i_{k;N_k})$ as long as we fix the last remaining index to one of the indices we have already chosen, we will still satisfy $\rho(\mathbf{i_k}) < u_k$. Then for each of the $u_k$ positions in the $u_k-tuple$ we may choose one out of $u_k-1$ candidates that we have chosen which gives $(u_k-1)^{u_k}$ possibilities. This covers also tuples with less than $u_k-1$ different indices as some of the candidates may not appear in the final tuple.   
Therefore, if $N_k \rightarrow \infty$, the expectation in Equation \ref{Eq:converge_icm} to $0$. Also note the number of possible tuples of indices $\mathbf{i_k}$ with $\rho(\mathbf{i_k}) = u_k$ is ${N_k \choose u_k}$. Hence the moment converges to:
\begin{equation*}
\begin{split}
    &  \lim_{\boldsymbol{N} \rightarrow \infty}m_{\boldsymbol{u}}(\nu_{\boldsymbol{N}}) = \mathbb{E}\Big[ \prod_{k=1}^K \prod_{n=1}^{u_k} \Xinkstar \Big]\\
\end{split}
\end{equation*}
The equality follows from Lemma \ref{lemma:unique_idx_icm} and we know $\lim_{N_k \rightarrow \infty} \frac{{N_k \choose u_k}}{N_k^{u_k}} = constant, \forall k$. Without loss of generality, consider the constant to be $1$, the remaining argument will not change. Using Theorem \ref{thm:mv_method_moments}, we have there exists a probability measure $\nu$ such that $m_{\mathbf{u}}(\nu) = \mathbb{E}\Big[ \prod_{k=1}^K \prod_{n=1}^{u_k} \Xinkstar \Big]$
\end{proof}

\begin{lemma}\label{lemma:finite_dim_equiv_icm}
      Let $\{\{A_{k;i}\}_{i \in \mathbb{N}}\}_{k=1}^K$  be $K$ infinitely joint exchangeable random binary processes. For any $K$ binary sequence $\{a_{k;1}, \ldots, a_{k;N_k}\}_{k=1}^K$ with $\sum_{n=1}^{N_k} a_{k;n} = r_k$:
    \begin{equation*}
    \begin{split}
      & \mathbb{P}(\{\{A_{k;n} = a_{k;n}\}_{n=1}^{N_k}\}_{k=1}^K)  \\
      & = \frac{1}{\prod_{k=1}^K {N_k \choose r_k}}~\mathbb{P}\Big(\sum_{n=1}^{N_1} A_{1;n} = r_1, \ldots, \sum_{n=1}^{N_K} A_{K;n} = r_K \Big)  
    \end{split}
    \end{equation*}
\end{lemma}

\begin{proof}[Proof of Lemma \ref{lemma:finite_dim_equiv_icm}]
We can distribute the 1's for each sequence in $\prod_{k=1}^K {N_k \choose r_k}$ different ways. Due to $K$ sequences being exchangeable, all of them have the same probability.
\end{proof}
Next, we will prove Causal conditional de Finetti.
\begin{proof}[Proof of Theorem \ref{thm:icm_cond_definetti}]
Let $\nu$ be Causal conditional de Finetti measure for $\{\{ \Xinkstar \}\}$ (see Definition \ref{def:icm_definetti_measure}) and define $K$ random binary processes $\{\{Z^k_n\}_{n\in \mathbb{N}}\}_{k=1}^K$ with the following finite dimensional distribution:
\begin{equation*}
\begin{split}
    &\mathbb{P}(\{\{Z^k_n = z^k_n\}_{n=1}^{N_k}\}_{k=1}^K) \\
  &= \int \prod_{k=1}^K \prod_{n=1}^{N_k} \pi_{\theta_k}(z^k_{n}) d\nu(\boldsymbol{\theta})
\end{split}
\end{equation*}
Note from the definition, the series $\{\{Z^k_{n}\}_{n\in \mathbb{N}}\}_{k=1}^K$ is infinitely joint exchangeable. We will prove  $\{\{X_{k;n}^{i,*}\}_{n\in \mathbb{N}}\}_{k=1}^K$ and $\{\{Z^k_{n}\}_{n\in \mathbb{N}}\}_{k=1}^K$ have the same finite dimensional distribution. Define the following random vector $\boldsymbol{R}$ where $R_k = \frac{1}{N_k} \sum_{n=1}^{N_k} Z^k_{n}$. By Lemma \ref{lemma:finite_dim_equiv_icm}, it suffices to show that $\mathbf{Q}$ (as in Definition \ref{def:icm_definetti_measure}) and $\mathbf{R}$ have the same distributions for all $N_k \in \mathbb{N}$ and for all $k$ and by the second statment in Theorem \ref{thm:mv_method_moments}, we know two probability distributions are identical if their moments agree. 

\begin{equation}
\label{eq:ordered_exp_icm}
\begin{split}
  &\mathbb{E}\big[\prod_{k=1}^K (Q_k)^{u_k}\big] =  \frac{1}{\prod_{k=1}^K N_k^{u_k}}\mathbb{E}\Bigg[\prod_{k=1}^K \Big(\sum_{n=1}^{N_k} \Xinkstar \Big)^{u_k} \Bigg] \\
  &=  \frac{1}{\prod_{k=1}^K N_k^{u_k}} \mathbb{E}\Bigg[\prod_{k=1}^K
  \Big(\sum_{\mathbf{i_k}} \prod_{n \in \mathbf{i_k}} \Xinkstar  \Big)
  \Bigg] \\
  &= \frac{1}{\prod_{k=1}^K N_k^{u_k}} \sum_{a_1=1}^{u_1} \ldots \sum_{a_K=1}^{u_K} \sum_{\substack{
  \forall k, \mathbf{i_k}: \\ \rho(\mathbf{i_k}) = a_k} }\mathbb{E}\big[ \prod_{k=1}^K \prod_{n=1}^{a_k} \Xinkstar \big]
\end{split}
\end{equation}
The last equality follows from Lemma \ref{lemma:unique_idx_icm}. 
From Theorem \ref{theorem:moment_expression}, we know the above expectations are in fact the moments of the probability measure $\nu$:
\begin{equation*}
\begin{split}
  &\mathbb{E}\big[ \prod_{k=1}^K  \prod_{n=1}^{a_k} \Xinkstar \big] = m_{\boldsymbol{a}}(\nu)\\
  &= \int \prod_k (\theta_k)^{a_k} d\nu(\boldsymbol{\theta}) \\
  &= \int \prod_{k=1}^K \prod_{n=1}^{a_k} \pi_{\theta_k}(x^{k, *}_{i;n}) d\nu(\boldsymbol{\theta})\\
  &= \mathbb{E}\big[ \prod_{k=1}^K \prod_{n=1}^{a_k} Z^k_{n} \big]
\label{eq:z_def_ones}
\end{split}
\end{equation*}
Hence continuing Equation \ref{eq:ordered_exp_icm} and reverting the steps taken in Equation \ref{eq:ordered_exp_icm} and using Lemma \ref{lemma:unique_idx_icm} we have:
\begin{equation*}
\begin{split}
  (\ref{eq:ordered_exp_icm}) &= \frac{1}{\prod_{k=1}^K N_k^{u_k}} \sum_{a_1=1}^{u_1} \ldots \sum_{a_K=1}^{u_K} \sum_{\substack{
  \forall k, \mathbf{i_k}: \\ \rho(\mathbf{i_k}) = a_k} }\mathbb{E}\big[ \prod_{k=1}^K \prod_{n=1}^{a_k} Z^k_n \big]\\
    &= \frac{1}{\prod_{k=1}^K N_k^{u_k}} \mathbb{E}\Bigg[\prod_{k=1}^K \Big(\sum_{n=1}^{N_k} Z^k_{n} \Big)^{u_k} \Bigg]\\
  &=\mathbb{E}\big[\prod_{k=1}^K (R_k)^{u_k}\big]\\
 \end{split}
\end{equation*}

Hence the moments of the joint distribution of $\boldsymbol{Q}$ and $\boldsymbol{R}$ are the same, therefore the joint distributions must agree.
\end{proof}

\begin{proof}[Proof of Causal de Finetti]
Recall $\{(X_{1;n}, X_{2;n}, \dots X_{d;n})\}_{n \in \mathbb{N}}$ is an infinite exchangeable sequence and satisfies condition 2 in Causal de Finetti. Without loss of generality, we reorder the variables according to reversed topological ordering, i.e. a node's parents will always be placed after the node itself. Note a reversed topological ordering is not unique, but it must satisfy a node's non-descendants will come before itself. Then by Kolmogorov's chain rule, we can always write the joint probability distribution as
\begin{equation}
    P(\{(X_{1;n}, X_{2;n}, \dots X_{d;n})\}_{n=1}^N) = \prod_{i=1}^d P(X_{i;:}|UP_{i;:})
\label{eq:bayesfac}
\end{equation}
Recall $UP_{i;:} := (X_{i+1;:}, \ldots, X_{d;:})$, which contains all random variables that have higher variable index value than $i$, i.e. upstream of node $i$. \\
For each $X_{i;:}$, we want to show there exists a suitable probability measure $\nu_i$ such that we can write
$P(X_{i;:}|UP_{i;:}) = \int \prod_{n=1}^N p(X_{i;n} | PA_{i;n}, \boldsymbol{\theta_i}) d\nu_i(\boldsymbol{\theta_i})$. This has been shown in Theorem \ref{thm:icm_cond_definetti}. 
Hence the joint distribution becomes:
\begin{equation*}
\begin{split}
& P(\{(X_{1;n}, X_{2;n}, \dots X_{d;n})\}_{n=1}^N) \\
& = \int \prod_{n=1}^N\prod_{i=1}^d p(X_{i;n} | PA_{i;n}, \boldsymbol{\theta_i}) d\nu_i(\boldsymbol{\theta_i})
\end{split}
\end{equation*}
and we complete the proof.
\end{proof}

\section{Proof of Identifiability result}
\label{sec:proof_identifiability}
\subsection{Identifiability under i.i.d}
Let's first see why under i.\,i.\,d regime, it is only possible to differentiate the causal structure up to a Markov equivalence class.

\begin{definition}[$\mathcal{I}$-map]
Let $P$ be a distribution, $\mathcal{I}(P)$ denotes the set of conditional independence relationships of the form $X \ind Y \mid Z$ that hold in $P$. Let $\mathcal{G}$ be a DAG, $\mathcal{I}(\mathcal{G})$ denotes the set of conditional independence assumptions encoded in $\mathcal{G}$ which can be directly read-off via d-separation \citep{pearl1988probabilistic}. 
\end{definition}

\begin{definition}[Bayesian network structure]
A Bayesian network structure $\mathcal{G}$ is a directed acyclic graph whose nodes represent random variables $X_1, \dots, X_n$. Let $PA_i^\mathcal{G}$ denotes the parents of $X_i$ in $\mathcal{G}$, and $ND_i^\mathcal{G}$ denotes the variables in the graph that are not descendants of $X_i$. 
\end{definition}

\begin{definition}[Global markov property]
    Given a DAG $\mathcal{G}$ and a joint distribution $P$, this distribution is said to satisfy \textbf{global markov property} with respect to the DAG $\mathcal{G}$ if $\mathcal{I}(\mathcal{G}) \subseteq \mathcal{I}(P)$ \citep{Pearl2009Causality:Inference.}. 
    Alternatively, we say $P$ is \textbf{Markovian} with respect to $\mathcal{G}$.
\end{definition}
\begin{definition}[Faithfulness]
Given a DAG $\mathcal{G}$ and a joint distribution $P$, $P$ is \textbf{faithful} to the DAG $\mathcal{G}$ if $\mathcal{I}(P) \subseteq \mathcal{I}(\mathcal{G})$ \citep{Pearl2009Causality:Inference.}. 

\end{definition}

We denote $\mathcal{M}(\mathcal{G})$ to be the set of distributions that are Markovian and faithful with respect to $\mathcal{G}$:
\begin{align*}
  \mathcal{M}(\mathcal{G}) := \{P: \mathcal{I}(P) = \mathcal{I}(\mathcal{G})\} 
\end{align*}
Two DAGs $\mathcal{G}_1, \mathcal{G}_2$ are Markov equivalent if $\mathcal{M}(\mathcal{G}_1) = \mathcal{M}(\mathcal{G}_2)$.
\begin{lemma}[Graphical criteria for Markov Equivalence \citep{Janzing2017ElementsAlgorithms}]
\label{lemma:graphical_criteria_for_markov_equivalence}
    Two DAGs $\mathcal{G}_1$ and $\mathcal{G}_2$ are Markov equivalent if and only if they have the same skeleton and the same v-structures. 
\end{lemma}
This means for any suitably $i.\,i.\,d$ generated distribution $P$, one cannot uniquely determine the underlying graph that generates this distribution but can only determine up to d-separtion equivalence. 
\subsection{Identifiability under exchangeable}

\begin{definition}[Acyclic Directed Mixed Graph (ADMG)]
    A graph $\mathcal{M}$ is acyclic if it contains no directed cycles, i.e a sequence of edges of the form $x \rightarrow \dots \rightarrow x$. There are two types of edge between a pair of vertices in ADMG: directed ($x \rightarrow y$) or bi-directed ($x \leftrightarrow y$). In particular, there could be two edges between a pair of vertices, but in this case at least one edge must be bi-directed to avoid a directed cycle. 
\end{definition}

\begin{definition}[ICM operator on a DAG]
   Let $U$ be the space of all DAGs whose nodes represent $X_1, \ldots, X_d$. Let $V$ be the space of ADMGs whose nodes represent $\{(X_{i;n})\}$, where $i \in [d], n \in \mathbb{N}$. A mapping $F$ from $U$ to $V$ is an ICM operator if $F(\mathcal{G})$ satisfies:
\begin{itemize}
    \item $F(\mathcal{G})$ restricted to the subset of vertices $\{X_{1; n}, \ldots, X_{d; n}\}$ is a DAG $\mathcal{G}$, for any $n \in \mathbb{N}$,
    \item $X_{i;n} \leftrightarrow X_{i;m}$ whenever $n \neq m$ for all $i \in [d]$
    \item there are no other edges other than stated above
\end{itemize}
We denote the resulting ADMG as $ICM(\mathcal{G})$. Let $\textbf{PA}_{i;n}^\mathcal{G}$ denote the parents of $X_{i;n}$ in $ICM(\mathcal{G})$ and similarly for $\textbf{ND}_{i;n}^\mathcal{G}$ for corresponding non-descendants.
\end{definition}

\begin{theorem}[Markov equivalence criterion for ADMGs \citep{pmlr-vR1-spirtes97b}]
\label{thm:markov_equiv_dmag}
Two ADMGs over the same set of vertices are Markov equivalent if and only if 
\begin{enumerate}
    \item They have the same skeleton
    \item They have the same v-structures
    \item If a path u is a discriminating path for a vertex B in both graphs, then B is a collider on the path in one graph if and only if it is a collider on the path in the other.
\end{enumerate}
Here we do not specify the details for condition 3 as it is not used in below proofs, for more details please refer to \citep{pmlr-vR1-spirtes97b}.
\end{theorem}
We inherit the argument used in finding correct causal structure in the i.i.d. regime where we match conditional independence assumptions encoded in the graph with observed conditional independence relationships in distributions. 
\begin{corollary}
\label{corollary_markov_equivalence_exchangeable}
\begin{align*}
    \mathcal{I}(\text{ICM}(\mathcal{G}_1)) = \mathcal{I}(\text{ICM}(\mathcal{G}_2)) \Leftrightarrow \mathcal{G}_1 = \mathcal{G}_2
\end{align*}
\end{corollary}

\begin{proof}
The direction for the case $\mathcal{G}_1 = \mathcal{G}_2$ is trivial. We are left to show that when $\mathcal{G}_1 \neq \mathcal{G}_2$, $\mathcal{I}(\text{ICM}(\mathcal{G}_1)) \neq \mathcal{I}(\text{ICM}(\mathcal{G}_2))$.

Suppose $\mathcal{G}_1, \mathcal{G}_2$ are BNs over the same set of random variables with length $n$. Suppose further $\mathcal{G}_1 \neq \mathcal{G}_2$, and is not markov equivalent, i.e. $\mathcal{I}(\mathcal{G}_1) \neq \mathcal{I}(\mathcal{G}_2)$. Be definition of $\text{ICM}(\mathcal{G})$, we have $\mathcal{I}(\mathcal{G}) \subseteq \mathcal{I}(\text{ICM}(\mathcal{G}))$ .Thus $\mathcal{I}(\mathcal{G}_1) \neq \mathcal{I}(\mathcal{G}_2)$ implies $\mathcal{I}(\text{ICM}(\mathcal{G}_1)) \neq \mathcal{I}(\text{ICM}(\mathcal{G}_2))$. 

Consider the case $\mathcal{G}_1 \neq \mathcal{G}_2$, but they are Markov equivalent. We know two DAGs are Markov equivalent if and only if they have the same skeleton and same v-structures. Thus, $\mathcal{G}_1$, $\mathcal{G}_2$ differs in some node $X_i$ where the orientation of its edge to some node $X_j$ is different while not deleting or creating new v-structures within $\mathcal{G}_1$ and $\mathcal{G}_2$. Wlog, let $X_{i;n} \rightarrow X_{j;n}$ in $\mathcal{G}_1$ and $X_{i;n} \leftarrow X_{j;n}$ in $\mathcal{G}_2$. Note $\text{ICM}(\mathcal{G}_1)$ and $\text{ICM}(\mathcal{G}_2)$ have different v-structures: $X_{i;n} \rightarrow X_{j;n} \leftrightarrow X_{j;n+1}$ is a v-structure in $\text{ICM}(\mathcal{G}_1)$, but not in $\text{ICM}(\mathcal{G}_2)$ as $X_{i;n} \leftarrow X_{j;n} \leftrightarrow X_{j;n+1}$. Thus by Theorem \ref{thm:markov_equiv_dmag}, result follows. 
\end{proof}
Denote $\mathcal{E}(\mathcal{G})$ to be the set of distributions that are Markovian and faithful to $\text{ICM}(\mathcal{G})$:
\begin{align*}
      \mathcal{E}(\mathcal{G}) := \{P: \mathcal{I}(P) = \mathcal{I}(\text{ICM}(\mathcal{G}))\}  
\end{align*}
Two DAGs $\mathcal{G}_1, \mathcal{G}_2$ are Markov equivalent under ICM generative process if $\mathcal{E}(\mathcal{G}_1) = \mathcal{E}(\mathcal{G}_2)$. By Corollary \ref{corollary_markov_equivalence_exchangeable}, $\mathcal{E}(\mathcal{G}_1) = \mathcal{E}(\mathcal{G}_2)$ if and only if $\mathcal{G}_1 = \mathcal{G}_2$.

\section{Algorithm}
\label{sec:appendix_algorithm}
\begin{algorithm}[!htp]
\DontPrintSemicolon
  \KwInput{For $e \in \mathcal{E}$, we have $(X^e_{1;n}, \ldots, X^e_{d;n})_{n=1}^{N_e}$ where $X^e_{i;n}$ denotes the $i$-th variable of $n$-th sample in environment $e$ and $N_e$ is the number of samples in environment $e$. Assume $N_e \geq 2, \forall e$.}
  \KwOutput{A directed acyclic graph $\mathcal{G}$}
  \KwStepone{Identify variables' topological ordering. Initiate index list $L := [1, \ldots, d]$ and an empty dictionary of lists $\mathbf{S}$ with keys from $L$. Set starting index $k = 1$.}
  \While{$L$ is not empty}
  {
    \For{$i\in L$}
    {
        \If{$X^e_{i;1} \ind X^e_{j;2} \mid \{X^e_{k;1}\}_{k\neq i}, \forall j\neq i$}
        {
        Append $i$ in $\mathbf{S}[k]$ and remove $i$ from $L$\;
        }
    }
    $k = k + 1$
  }
  \KwSteptwo{Identify edges. Set $t:=1$ and reset $k:=1$. Here since each sample shares the same underlying causal graph, we abbreviate $X_{i;n}$ with $X_i$ where $n$ can be any number.}
  \While{$t \leq d-1$}
  {
    \While{$k \leq d-t-1$}
    {
    \For{each $i \in S_k$ and $j \in S_{k+t}$}
        {
        \If{$t=1$}
            {
            Test $X_{i} \ind X_{j} \mid \bigcup_{m > k+t}\bigcup_{i \in S_m} X_i$, if it holds then there exist no edge, else $X_{j;n} \rightarrow X_{i;n}, \forall n$.
            }
        \If{$t > 1$}
            {
            Define Z contains three set of nodes: $ \bigcup_{m > k+t}\bigcup_{i \in S_m} X_i$, $\textbf{PA}_i \cap S_{<{k+t}}$, $S_{k+t} \backslash X_j$.
            Test $X_{i} \ind X_{j} \mid Z$, if it holds then there exists no edge, else $X_{j;n} \rightarrow X_{i;n}, \forall n$.
            }
        }
    }
  }
\caption{"Causal-de-Finetti" Algorithm: causal discovery in ICM-generative processes}
\label{alg:mvcausal_exchangeable}
\end{algorithm}

\subsection{Proof}
Assume there exists no unobserved latent variables and our observed data is indeed generated from some ICM generative process. Algorithm \ref{alg:mvcausal_exchangeable} can identify the underlying DAG. Below we show its main steps. 
\begin{enumerate}
    \item Identify the topological ordering of observed variables
    \item Identify edges between different topological orders. 
\end{enumerate}

\begin{definition}[$n$-order sinks]
    Given a DAG, $X_{i;n}$ denotes the $n$-th sample in one environment and $i$-th variable in $n$-th sample: 
    \begin{itemize}
        \item A node $X_{i;n}$ is a first-order sink $S_1$ if it does not have any outgoing edges.
        \item A node $X_{i;n}$ is an $k+1$-order sink $S_{k+1}$, if all of its outgoing edges are to $l$-order sink (where $l < k+1$) and at least one of them is a $k$-order sink.
    \end{itemize}
    
    We denote the set of $k$-order sinks as $S_k$ and $\cup_{i=1}^{k-1}S_i = S_{<k}$.
\end{definition}

In Step 1 of the algorithm, we aim to use appropriate conditional independence tests to determine the topological ordering of our observed variables. The main idea is we iteratively find the first-order sinks $S_1$ and then remove them to find the next first-order sinks, and so on. 

\leafnodeidentification*
\begin{proof}[Proof of Lemma \ref{lemma:leaf-node-identification}]
Let $X_{i;n} \in S_1$. We first examine all the possible paths between $X_{i;n}$ and $X_{j;m}$. There are two cases for any such paths: it starts by $X_{i;n} - X_{k;n} - \ldots$ for some $k \neq i$, or $X_{i;n} - \theta_i - X_{i;p}$ for some $p \neq n$. In the first case, since $X_{i;n}$ is a first-order sink, the first edge is outgoing from $X_{k;n}$ and hence is blocked by conditioning on $X_{k;n}$. In the second case, we cannot continue from $X_{i;p}$ since it does not have any outgoing edges and the path would then include a collider and we thus did not condition on $X_{i;p}$.

To prove the converse, assume that $X_{i;n}$ is not a first-order sink but still satisfies the conditional independence. However, this would mean that $X_{i;n}$ has an outgoing edge to $X_{k;n}$ for some $k \neq i$. Then the path $X_{i;n} \xrightarrow{} X_{k;n} \xleftarrow{} \theta_k \xrightarrow{} X_{k;m}$ is activated by conditioning on $X_{k;n}$, and hence the conditional independencies cannot hold.
\end{proof}

This lemma provides us with a test to search for first-order sinks. We can state a similar lemma for $k$-order sinks (note that conceptually the lemma is equivalent to iteratively find first order sinks after removing the original lower order sinks).

\begin{lemma}[$k$-order sink condition]\label{lemma:kordersink}
    A node $X_{i;n}$ is a $k$-order sink if and only if the following holds for every $m\neq n$ and $j \neq i$ and $X_{j;m} \in S_{\geq k}$:
    \begin{equation*}
    \label{norderCI}
        X_{i;n}\ind X_{j;m} ~\vert~ \{X_{l;n}\}_{l\neq i} - S_{<k}
    \end{equation*}
\end{lemma}

Using Lemma \ref{lemma:leaf-node-identification} and \ref{lemma:kordersink}, we can iteratively determine the set of $S_k$ for all $k$. Note, that these sets represent a topographic ordering of observable variables: edges can only run from higher order sinks to lower order sinks. Hence, the only thing remaining is to determine for every $k$ and every pair of nodes $X_{i;n} \in S_k$ and $X_{j;n} \in S_{<k}$ if they are connected by an edge. Since each sample shares the same underlying causal structure, without loss of generality, we abbreviate $X_{i;n}$ with $X_i$. 


\edgeidentification*

\begin{proof}[Proof of Lemma \ref{lemma:edgeidentification}] 
First, we prove direction ($\Rightarrow$) and suppose there is no direct edge, i.e. $X_i \to X_j$. If there is no path connecting $X_i$ and $X_j$, then $X_i \ind X_j$. The result trivially holds. Next, we assume there is a path $p$ connecting $X_i$ and $X_j$. The path must satisfy one of the two conditions: Either (case 1) there exists $X_k$ in the path $p$ such that $X_k \in S_{>n}$, or (case 2) all variables in the path $\in S_{\leq n}$. 

When $k = 1$: 
    Under case 1), let $W$ be the set containing all variables belong to $S_{>n}$. Note $\lvert W \rvert \geq 1$ by condition. Then there exists a non-collider $X_k \in W$. Suppose all variables in $W$ are colliders and $W \neq \emptyset$, then there must $ \exists X_l \in p$ and $X_l \not \in W$ that has an edge outgoing to some variable contained in $W$. Since $W \neq \emptyset$, call that variable as $X_w$ and $X_w \in S_{>n}$. By definition of sink orders, $X_l \in S_{>n+1}$, $X_l \in W$. So $W$ is incomplete. Contradiction. 
    Under case 2), we show that there exists a collider in the path. Suppose there is no collider on the path, then all the edge are directed edges in one direction. Since edges can only go from higher sink orders to lower sink orders, the edges can only go from $X_i$ to $X_j$ in one direction. Since the path is also not a 1-arrow direct path, there $\exists X_a \in p, a \neq \{i, j\}$. Choose $X_a \in \textbf{PA}_j$, since it has a 1-arrow direct path to $X_j$, $X_a \in S_{\geq m + 1}$. Further since $X_i$ is an ancestor of $X_a$, $X_i \in S_{\geq m+2}$. Here $n = m+1$, so $X_i \in S_{\geq n+1}$. Contradicts the condition that $X_i \in S_n$. 
    Given the path connecting $X_i$ and $X_j$ either contains variables in higher sink orders and there exist a non-collider $X_k \in p$ and $X_k \in S_{>n}$., such path can be blocked by conditioning on the set $S_{>n}$; or the path only contains variables in $S_{\leq n}$ and the path must exist a collider, which we block the path by not conditioning on any variables in $S_{\leq n}$. 
    Hence the proposed conditional independence holds by d-separation. By Markov assumption, the conditional independence also holds in distribution.  
    
when $k > 1$:
    Similarly, under case 1), by same argument as above, there exist a non-collider $X_k \in p$ and $X_k \in S_{>n}$. Therefore conditioning on all variables in $S_{>n}$ blocks the set of paths under case 1. Under case 2) when all variables in the path belong to $S_{\leq n}$, then the parent of $X_j$ in this path, here we call it $X_p$, either (case 2.1) $X_p \in \textbf{PA}_j \cap S_{<n}$, or (case 2.2) $X_p \in \textbf{PA}_j \cap S_n$. Under case 2.1, $X_p$ is a non-collider in the path, since it has one outgoing edge. Then conditioning on $\textbf{PA}_j \cap S_{<n}$ blocks the set of paths under case 2.1. Under case 2.2), for any variable $X_m \in S_n$ on the path $p$, $X_m$ only has outgoing edges since all variables on the path belong to $S_{\leq n}$. $X_m$ is also a non-collider, thus conditioning on $S_n$ blocks the set of paths satisfying case 2.2). 
    Similarly, the proposed conditional independence holds by d-separation. By Markov assumption, the conditional independence also holds in distribution. 

Finally, we prove the other direction ($\Leftarrow$): suppose the proposed conditional independence holds, then there does not exist a 1-arrow direct path from $X_i$ to $X_j$. Suppose there exists a 1-arrow directed path and conditional independence holds, then conditioning on the proposed set does not d-separate the path. By faithfulness, conditional independence does not hold in distribution and the result follows. \looseness=-1
\end{proof}

\section{Relation to causality in time-series}
\label{sec:time-series}

Recall temporal SCMs: let $\mathbf{V}_t = (V_t^1, \ldots, V_t^N)$ represent the dynamic process variables underlying a multivariate time series. The structural assignments are: 
\begin{align}
    V_t^j := f^j(pa(V_t^j), \eta_t^j), \forall \ V_t^j \in \mathbf{V}_t \ \text{and} \ t \in \mathbb{Z}
\end{align}
with jointly independent random variables $\eta_t^j$. The causal parents $pa(V_t^j)$ are direct causes and a subset of $\{\mathbf{V}_t, \ldots, \mathbf{V}_{t - \tau_{max}}\} \backslash \{V_t^j\}$ with $\tau_{max} \geq 0$.

The non-negative integer $\tau_{max}$ means sequential data has directional influence, i.e., no future events could influence the current. Zero is included, as there may be contemporaneous causal influences $V_t^i \to V_t^j$. One could assume such SCM is causally stationary, i.e., the causal relationships and noise distributions are assumed to be invariant in time. 

Consider an ICM-generative process, the sample index is exchangeable, i.e., the order of observation does not matter. This is in contradiction to causality in time series as only earlier or contemporaneous samples could constitute potential causal parents of one variable. 

When one interprets the sample index in the ICM-generative process as time steps, for example in bivariate cases. $X_t \to Y_t, X_{t+1} \to Y_{t+1}, X_t \leftrightarrow X_{t+1}, Y_t \leftrightarrow Y_{t+1}, \forall t$. Firstly, it follows a degree of causally stationary assumption, i.e., the causal relationships are assumed to be invariant in time. Secondly, all sequences of time series have unobserved confounders that influence all the variables in time series. Sometimes, one may interpret such temporal SCMs as the out-of-variable problem, i.e. one lacks the observation of the unobserved confounder $\theta_t$, if it exists in the physical world. Suppose we observe $\theta_t, \psi_t, \forall t$, then ICM-generative process with sample index as time-steps can be rewritten as $\theta_t \to X_t \to Y_t \leftarrow \psi_t, \theta_{t+1} \to X_{t+1} \to Y_{t+1} \leftarrow \psi_{t+1}, \theta_t \to X_{t+1}, \psi_t \to Y_{t+1}$, with $\theta_t = \theta_0, \psi_t = \psi_0, \forall t$. Thus, it follows the temporal SCM formulation. It suggests data with less structure can be modelled by formulations for data with more complex structure with the right instantiation - this is the case for exchangeable sequences with i.i.d.\ data, and so is the case for time-series data with exchangeable data. 


\end{document}